\documentclass[journal]{IEEEtran}

\ifCLASSINFOpdf

\else

\fi

\usepackage{amsmath,graphicx}
\usepackage{amsmath,graphicx}
\usepackage{amsmath,amsfonts,amssymb,amsthm, bm}
\usepackage{epstopdf}
\usepackage{color}
\usepackage{lipsum}

\usepackage{url}

\DeclareMathOperator{\hist}{hist}
\DeclareMathOperator*{\argmin}{arg\,min}

\DeclareMathOperator{\Poisson}{Poisson}

\DeclareMathOperator{\T}{\top}
\DeclareMathOperator{\E}{\mathbb{E}}
\DeclareMathOperator{\supp}{Supp}

\newtheorem{theorem}{Theorem}
\newtheorem{corollary}{Corollary}
\newtheorem{lemma}{Lemma}
\newtheorem{definition}{Definition}

\allowdisplaybreaks[3]

\begin{document}
%
\title{Minimax Optimal Sparse Signal Recovery with Poisson Statistics}

\author{
  \IEEEauthorblockN{Mohammad H. Rohban, Delaram Motamedvaziri, and Venkatesh Saligrama\\}
  \IEEEauthorblockA{
    Electrical and Computer Engineering Department Boston University}}

\maketitle

\begin{abstract}
 We are motivated by problems that arise in a number of applications such as Online Marketing and Explosives detection, where the observations are usually modeled using Poisson statistics. We model each observation as a Poisson random variable whose mean is a sparse linear superposition of known patterns. Unlike many conventional problems observations here are not identically distributed since they are associated with different sensing modalities. We analyze the performance of a Maximum Likelihood (ML) decoder, which for our Poisson setting involves a non-linear optimization but yet is computationally tractable. We derive fundamental sample complexity bounds for sparse recovery when the measurements are contaminated with Poisson noise. In contrast to the least-squares linear regression setting with Gaussian noise, we observe that in addition to sparsity, the scale of the parameters also fundamentally impacts $\ell_2$ error in the Poisson setting. We show tightness of our upper bounds both theoretically and experimentally. In particular, we derive a minimax matching lower bound on the mean-squared error and show that our constrained ML decoder is minimax optimal for this regime.
\end{abstract}

\IEEEpeerreviewmaketitle


\section{Introduction}
\IEEEPARstart{I}{n} this paper, we study sparse signal recovery problem with Poisson observations. This problem is motivated by many practical applications where the observations are the counts of an event. The mean count in these applications depends linearly on a sparse subset of parameters. Our goal is to extract the sparse subset from a potentially large number of parameters. Some of the practical applications motivating our problem include explosive identification based on photon counts in fluoroscopy \cite{cite6}, and eMarketing based on website traffic \cite{cite13}.


Motivated by these applications we propose a high-dimensional sparse signal estimation problem governed by Poisson distributed observations. We model the rate of the Poisson process as a positive mixture of known signatures. The rates for different observations are obtained from heterogeneous sensors or different measurement settings and therefore the observations are not identically distributed.  
Specifically, we model observations $y_i$ as follows.
$$\forall i \in \{ 1,\ldots,n \} : y_i\sim \Poisson(\lambda_{0,i}+{\mathbf a}_i^{\T}{\mathbf w}^\star)$$
where $\lambda_{0,i}$ is the rate of the  background Poisson noise and assumed to be known. Each ${\mathbf a}_i=[a_{i,1}, \ldots ,a_{i,p}]^{\T}$ is a distinct vector corresponding to the $i$-th sensor. The collection of these vectors form the sensing matrix, ${\mathbf A} = [{\mathbf a}_1,\ldots, {\mathbf a}_n]^{\T}$. Our goal is to recover the sparse vector, ${\mathbf w}^\star$, from $\{ y_1, \ldots, y_n \}$. The parameters ${\mathbf w}^\star \in \mathbb{R}_+^p$ are the mixing proportions for different sensors. 

Our high-dimensional setting entails $p\gg n$ and the ground truth mixture parameter ${\mathbf w}^\star$ is assumed to be sparse with $\ell_0$ norm $\|{\mathbf w}^\star\|_0=k$.
%
%
We derive upper and lower bounds for sparse signal recovery for this Poisson setting. Our upper bounds are based on analyzing the performance of an $\ell_1$ constrained maximum-likelihood estimator. The optimal ML decoder is obtained by solving a convex optimization problem involving a non-linear objective function. We then analyze performance of minimax estimators, namely, optimal estimators for worst-case mean-squared error to obtain lower bounds. These lower bounds are derived by means of Fano bounds and are obtained by embedding the estimation problem into a detection setting by utilizing Gilbert-Varshamov bound \cite{Gun11}. Specifically, for a given sparsity level $k := \lVert {\mathbf w}^\star \rVert_0$, and a parameter amplitude $s := \lVert {\mathbf w}^\star \rVert_1$ we show that the upper bound for the estimation error scales as $\sqrt{s k/n}$. We then obtain a matching lower bound thus demonstrating the minimax optimality of our $\ell_1$ constrained ML estimator. 



\noindent
{\bf Dependence on Scale:}
The dependence of mean-squared error on scale parameter $s$ does not arise in conventional sparse linear least squares regression setting. Indeed sample complexity is primarily determined by sparsity for many random matrix designs and sample complexity improves with scale of the ground truth parameter ${\mathbf w}^\star$ for a fixed level of noise~\cite{aeron10}. In contrast, in our problem setting, sample complexity degrades with the scale of the parameter vector. 
Intuitively this results from the fact that the curvature of the log likelihood function decreases with scale $s$ of the parameter vector. Indeed, for large values of $s$, partial changes $\partial \widehat {\mathbf w} = {\widehat{\mathbf w}} - {\mathbf w}^\star$ in the parameter vector translate to significantly smaller changes of the likelihood function resulting in lack of identifiability. Consequently, unlike the conventional case we inevitably have to suffer the effects of scale for the Poisson case. 



Another aspect of the problem is the high dimensional sparse setting, which leads to fundamental issues such as the Hessian being singular \cite{cite20,Sara}. To overcome these issues we follow along the lines of sparse linear regression and consider optimizing the Poisson likelihood function under $\ell_1$ constraints. These constraints have the effect of constraining the error patterns to a cone of feasible directions (descent cone):
\begin{equation} \label{OrigCone}
\widehat{\mathbf w} - {\mathbf w}^\star \in \mathbb{C} := \{{\mathbf u} : \|{\mathbf u}_{S^c}\|_1 \leq \|{\mathbf u}_S\|_1, |S| \leq k\}
\end{equation}
As a result it turns out that we need to ensure that the loss function is ``well-behaved'' only on the feasible cone. A recent important development~\cite{cite10} in this context is to impose strong convexity of the loss function on the feasible cone. Given suitable assumptions on the design matrix ${\mathbf A}$, this requirement is generally satisfied for many loss functions including least-squares losses. 

Unfortunately as it turns out, our Poisson case does not lead to strongly convex losses. Strong-convexity of the loss function amounts to the assumption that we see a non-trivial change in the loss function as a result of underlying parameter variation regardless of the ground truth ${\mathbf w}^\star$. This requirement can be viewed as a need for the curvature of the loss function to be non-vanishing on the cone. In the Poisson case, the perturbation in the loss function behaves linearly in large $s$ regimes (i.e. the curvature vanishes in the limit) and so the loss function is no longer strongly convex on the cone. This issue motivates us to incorporate signal strength $s$ in the definition of plausible error patterns $\mathbb{C}$.
We then characterize the change in loss function at various amplitudes and sparsity levels in terms of the changes in the parameter vector ($\|\partial \widehat {\mathbf w}\| = \epsilon$). We show that the loss function is strongly convex on the newly defined set of plausible error directions, as long as the elements of ${\mathbf A}$ are bounded and ${\mathbf A}$ satisfies the so-called {\it Restricted Eigenvalue (RE)} condition \cite{cite10}. 
A wide variety of sensing matrices ranging from deterministic to random designs can be shown to satisfy RE condition.
In particular our results apply to both deterministic and random sensing matrices and we present several results for both cases. We also conduct several synthetic and real-world experiments and demonstrate the  tightness of the oracle bounds on error as well as the efficacy of our method. Specifically, it has been suggested in the literature that LASSO can handle exponential family noise such as that arises in our application~\cite{cite5}. It turns out that $\ell_1$ constrained ML estimator uniformly outperforms LASSO and has significantly superior performance in many interesting regimes.

The paper is organized as follows: In Section \ref{PS}, we introduce the notation and state our sparse estimation problem. Section \ref{OA} describes our theoretical results on the convergence of the regularized ML decoder. The numerical results for different interesting scenarios are demonstrated in Section \ref{NR}. Finally, the detailed proof of the main theorems and lemmas are provided in Section \ref{AP}.

\subsection{Related Work}

Parameter estimation for non-identical Poisson distributions has been studied in the context of Generalized Linear Models (GLMs). However, our model is inherently different from the exponential family of GLM models that has been studied in \cite{cite10,cite12,cite14,cite15}. In particular the GLM model corresponding to the Poisson distributed data studied in the literature has the following form:
\begin{equation*}
\text{Model I : }  ~ y_i \sim \Poisson(\exp\left({\mathbf a}_i^{\T}{\mathbf w}^\star\right))
\end{equation*}
Therefore, the log likelihood takes the following form:
\begin{align*}
\mathcal{L}_1({\mathbf w})=\sum_{i=1}^n y_i \left({\mathbf a}_i^{\T} {\mathbf w}\right)-\exp\left({\mathbf a}_i^{\T}{\mathbf w}\right)
\end{align*}
In contrast, in the setting we are interested in, the observations are modeled as follows:  
\begin{equation*}
\text{Model II : } y_i \sim \Poisson(\lambda_{0,i}+{\mathbf a}_i^{\T} {\mathbf w}^\star )
\end{equation*}
and the log likelihood function has the form:
\begin{align*}
\mathcal{L}_2({\mathbf w})=\sum_{i=1}^n y_i \log \left(\lambda_{0,i}+{\mathbf a}_i^{\T} {\mathbf w}\right)-\left(\lambda_{0,i}+{\mathbf a}_i^{\T}{\mathbf w}\right)
\end{align*}

There are several important differences between the two models. We observe that imposing sparsity on ${\mathbf w}$ in Model I corresponds to smaller number of multiplicative terms in the Poisson rates. On the other hand, ${\mathbf w}$ being sparse in Model II results in fewer number of additive terms in the Poisson rates of the corresponding model. 
At a more fundamental level the loss function (negative log-likelihood) for Model I has an exponential term ($\exp\left({\mathbf a}_i^{\T} {\mathbf w} \right)$). The assumptions of strong convexity on the feasible cone $\mathbb{C}$ could be proved if elements of ${\mathbf A}$ are independent sub-gaussian draws \cite{cite10}. Consequently, unlike our case, the issue of signal amplitude no longer arises for this model. 
%
%
%

We can view model I as an instance of a general class of sparse recovery problems. Indeed, \cite{cite14} studies the convergence behavior of $\ell_1$ regularized ML estimation for exponential family distributions and GLM in this context. The bounds on error for sparse recovery of the parameter are based on the RE condition. Moreover, in order to get useful bounds on estimation error of GLM, they additionally need the natural sufficient statistic of the exponential family to be sub-gaussian. This condition could clearly be violated in our setting where the data is Poisson distributed and there is no constraint on the sensing matrix to be sub-gaussian.

More generally, \cite{cite10} describes a unified framework for  analysis of regularized $M-$ estimators in high dimensions. They also mention extension of their framework to GLMs and describe ``strong convexity" of the objective function on $\mathbb{C}$ as a sufficient condition to obtain consistency of M-estimators under Model I. As described earlier, this requirement of strong convexity on $\mathbb{C}$ is not consistent with our model. In addition the statistical aspects in that work requires that the components of the sensing matrix be characterized by \textit{sub-gaussian distributions}, which we do not require here. 



Statistical guarantees for sparse recovery in settings similar to model II have been provided in \cite{cite11,cite112,cite113} in the context of photon limited measurements. They assume that the observations are distributed as follows  
$$y_i\sim \Poisson({\mathbf a}_i^{\T}{\mathbf w}^\star)$$
where elements of the signal ${\mathbf w}^\star$ and sensing matrix are positive, and the sensing matrix satisfies the so-called Flux Preserving assumption:
$$ \sum_{i=1}^{n} {\mathbf a}_i^\top {\mathbf w}^\star \leq \| {\mathbf w}^\star \|_1 .$$

The latter assumption arises in some photon counting applications, like imaging under Poisson noise, where the total number of expected measured photons cannot be larger than the intensity of the original signal. The upper bound on reconstruction error of the constrained ML estimator is given in the paper \cite{cite112}. Surprisingly, the upper bound scales linearly with the number of measurements. However, this sounds reasonable under the Flux Preserving assumption. In fact this behavior is due to the fact that for a fixed signal intensity, more measurements lead to lower SNR for each observation. As a result, unlike conventional compressive sensing bounds, the estimates do not converge to the ground truth with increasing the sample size.
Nevertheless, Flux Preserving constraint does not arise in our setting and consequently the application and methods of analysis are different.

In summary the fundamental differences in the underlying model as well as the different assumptions in the sensing matrices from the previous work warrants new analysis techniques which is the subject of this paper.  



\subsection{Applications}

In the sequel, we will introduce two applications, which motivate the model described earlier: \\ 

{\bf 1. Explosive Identification:} In the explosive identification example, an unknown mixture of explosives is exposed to different  fluorophores. The goal is then to estimate the mixture components based on the observed fluorophores photon counts. Here $y_i$ and $\lambda_{0,i}$ could be considered as the photon counts and background emission rates for fluorophore $i$. $\lambda_{0,i}$ is measured before the exposure of the explosive mixture and is known. $a_{ij} < 0$ represents a quenching effect of explosive $j$ on fluorophore $i$, and $w^\star_j$ is the weight of explosive $j$ in the mixture~\cite{cite6}. \\

{\bf 2. eMarketing:} In the eMarketing example~\cite{cite13}, weekly traffic of different websites within the same market are measured. The traffic of site $i$, denoted by $y_i$, is assumed to be affected by the number of bought links to it on different advertisement websites. Each advertisement website $j$ is also assumed to contribute to the visiting rate by a fixed weight $w^\star_j$. These weights could be viewed as measures of popularity/dominance of advertisement website $j$. Moreover, $\lambda_{0,i}$ is the average traffic that visits website $i$ directly (not through intermediate advertisement website) and is acquired through online statistics of the website. $a_{i,j} > 0$ is the number of backward links that the business website $i$ has bought from the advertisement website $j$. Estimation of $w^\star_j$'s can lead to discovery of dominant advertisement websites in some industry. Fig. \ref{MarketingFig} illustrates this eMarketing model.

\begin{figure}
\begin{minipage}[b]{1\linewidth}
  \centering
  \centerline{\includegraphics[width= 8cm]{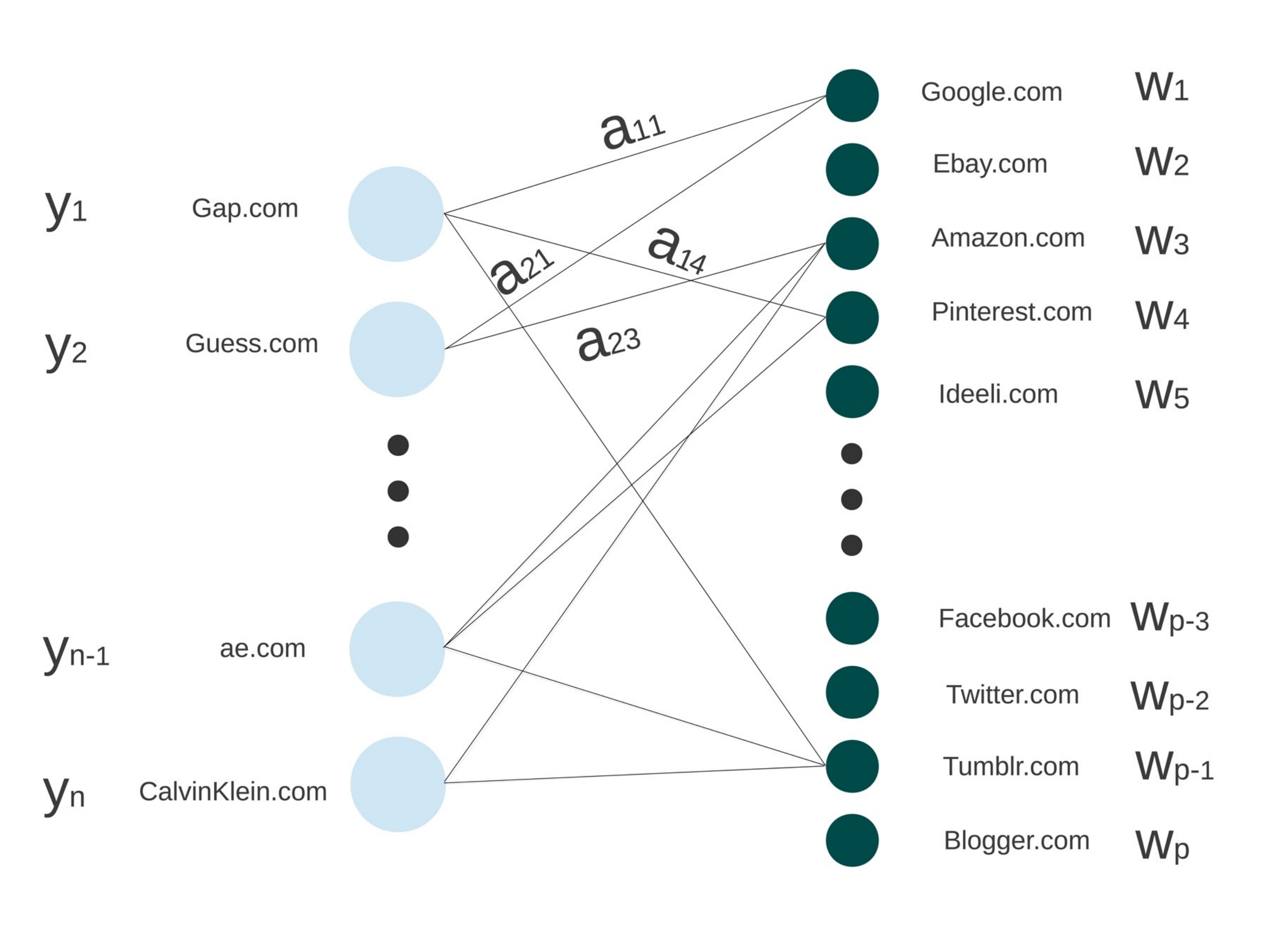}}
\caption{ Illustration of eMarketing model. Node on the left are the business websites, Nodes on the right are the advertisement websites. A connecting edge, $a_{ij}$, is the number of backward links purchased by the business website $i$ from the advertisement website $j$.
}
\label{MarketingFig}
\end{minipage}
\end{figure}


\section{Definitions and Problem Setup} \label{PS}
\noindent
{\bf Problem Definition}:
Let $y_i \sim \Poisson(\lambda_i({\mathbf w}^\star))$, with $\lambda_i({\mathbf w}) := \lambda_{0, i} + {\mathbf a}_i^{\top} {\mathbf w}$ and \hbox{${\mathbf a}_i, {\mathbf w}^\star \in \mathbb{R}^{p}_{+}$.} Furthermore, $\| {\mathbf w}^\star \|_1 = s$ and $ \| {\mathbf w}^\star \|_0 = k $. The main objective of this paper is to analyze the $\ell_2$ error of the maximum likelihood estimation of ${\mathbf w}^\star$ given ${\mathbf y} := (y_1, \ldots, y_n)$ and $s$. \\

\noindent
{\bf Likelihood Function}:
Let the objective function be the negative log likelihood of the data given the unknown parameter ${\mathbf w}$:
 $$Q({\mathbf w}) := \frac{1}{n} \sum_{i = 1}^{n} -y_i \log\lambda_i({\mathbf w}) + \lambda_i({\mathbf w}).$$ 
We will let $\lambda_{\min} := \min_{i} \lambda_{i}({\mathbf w}^\star) $, $\lambda_{\max} := \max_i \lambda_{0, i} + a_{\max} s$, $\overline{\lambda}_h := (1/n \sum_{i=1}^n 1/\lambda_i({\mathbf w}^\star))^{-1}$ (the harmonic average of the rates), $a_{\max} := \max_{i,j} {a}_{i,j} $, and ${\mathbf A} := [ {\mathbf a}_1, \ldots, {\mathbf a}_n]^\top$ throughout the paper. 

The guarantees on the error of the maximum likelihood estimator relies on an assumption on the sensing matrix $\mathbf A$ known as {\it Restricted Eigenvalue Condition}:

\begin{definition} \label{RECond} {{\bf Restricted Eigenvalue condition}}:
A matrix ${\mathbf A} \in \mathbb{R}^{n \times p}$ satisfies RE$(k, \gamma_k)$, if
there exists a constant $\gamma_k > 0$, such that for any set of indices $S$ satisfying $|S| \leq k$, and vector 
$${\mathbf u} \in \mathbb{C}(S) := \{ {\mathbf u} \neq {\mathbf 0} : \lVert {\mathbf u}_S \rVert_1 \geq \lVert {\mathbf u}_{S^c} \rVert_1 \},$$
we have:
\begin{equation} \label{REC}
 \frac{1}{n} \| {\mathbf A} {\mathbf u}\|_2^2\geq \gamma_k \| {\mathbf u} \|^2_2
\end{equation}
where ${\mathbf u}_S$ is the restriction of the vector ${\mathbf u}$ to the indices in $S$, and {$S^c = \{1, \ldots, p \} \setminus S$.}
\end{definition}

\section{Main Results}
\label{OA}

\subsection{Upper Bound on $\ell_2$ Error}
We will analyze the performance of $\ell_1$ constrained Max-Likelihood estimator to obtain upper bounds. Specifically, we consider the following estimator:
\begin{equation} \label{eq:opt1}
\widehat{\mathbf w} := \argmin_{\sum_{i} {w}_i \leq s, ~ \forall i ~:~ {w}_i \geq 0} Q({\mathbf w})
\end{equation}

\begin{theorem} \label{thm:main}
Supose ${\mathbf A}$ satisfies RE$(k, \gamma_k)$, and its elements are positive and bounded by $a_{\max}$. Then, with probability at least $1 - \zeta$, the following error bound holds:
$$ \|  \widehat{\mathbf w}  - {\mathbf w}^\star \|_2 \leq  \frac{54 \lambda_{\max} a^2_{\max} \left(  3 +  \log\frac{\lambda_{\max}}{\lambda_{\min}}  \right)}{\gamma_k} \sqrt{\frac{k \log(2/\zeta)}{\overline{\lambda}_h n}},$$
where $\zeta \geq 2 \exp\left(\frac{-n \lambda_{\min} \min(1, \lambda_{\min}) }{4 \overline{\lambda}_h }\right)$.
\end{theorem}

\noindent 
{\bf Mean Squared Error:} Note that since the rates are uniformly bounded in terms of the scale parameter $s$ and the error probability can be made arbitrarily small, we can obtain an upper bound for the mean squared error (MSE) in a straightforward manner. The MSE upper bounds has an identical scaling with respect to sparsity, scale and the number of measurements.

\noindent
{\bf Upper Bound for Special Cases}: Consider the case where base rates $\lambda_{0, i}$ are constant and so $\lambda_{\max} = \mathcal{O}(s)$. For this scenario we have with probability $1-\zeta$ that,
$$\|  \widehat{\mathbf w}  - {\mathbf w}^\star \|_2 = \mathcal{O}\left ({s \log s \over \gamma_k } \sqrt{{k \log(2/\zeta)}/{n}}\right )$$ 
In addition if we also have $\lambda_{\min} \geq c \lambda_{\max} $ for a constant $c$, so that the harmonic mean $\overline{\lambda}_h = \mathcal{O}(s)$ and $\log \lambda_{\max} / \lambda_{\min} = \mathcal{O}(1)$, then, it follows that the error behaves as 
$$\|  \widehat{\mathbf w}  - {\mathbf w}^\star \|_2 = \mathcal{O}\left ({\sqrt{{s k \log(2/\zeta)}/{n}}  \over \gamma_k} \right  )$$

\noindent
{\bf Other Distributions}: The techniques in the proof could be used to extend this result to other heavy tailed distributions such as Gamma distribution, when the scale parameter is an affine function of ${\mathbf w}^\star$. Similar to the Poisson case, the curvature of the loss function associated with the log-likelihood of Gamma distribution vanishes with scale $s$. 

\noindent
{\bf Additional Linear Constraints}: The proof could be easily modified when there are additional {\it linear} constraints on the $\lambda_i({\mathbf w})$ values. In that case, the constraints should be included in the set $\mathbb{C}_s$. Hence, the results can be generalized to the case where $a_{i,j} < 0$, so long as the Poisson rates are constrained to be positive. 

\noindent
{\bf Brief Proof Outline}: The main idea of our approach is that we can show that the $\ell_2$ error can be bounded by $\epsilon$ if no feasible perturbation of magnitude $\epsilon$ around the true parameter value can lead to a decrease in the objective function $Q$. Furthermore, if $Q$ is strongly convex around ${\mathbf w}^\star$, an appropriate upper bound on $\epsilon$ makes the change in the objective function positive for all plausible error patterns. Finally, the upper bound on $\epsilon$ also yields an upper bound on the $\ell_2$ error. As the loss function associated with Poisson distributed data is not strongly convex on $\mathbb{C}$, we replace the cone $\mathbb{C}$ (Eq. \eqref{OrigCone}) by another set $\mathbb{C}_s$, which depends on $s$. This leads to a loss function that is well-behaved and subsequently we can derive our results.

\subsection{Sparsity Scaling}

Theorem~\ref{thm:main} describes an upper bound for estimation error for matrices belonging to RE$(k, \gamma_k)$. Since $\gamma_k$ in turn depends on sparsity $k$, the dependence of estimation error on sparsity is implicit. One difficulty in making this dependence explicit is that RE condition for matrices that are positive and bounded has not been extensively studied. The best known bound is due to \cite{Rud10} for positive and bounded matrices, who shows matrix constructions with $n=\Omega(k^2 \log p)$ satisfying RE condition with a constant $\gamma_k$ that is independent of $k$. This in turn implies that for positive bounded matrices it is sufficient to have $n=\Omega(k^2 \log p)$ measurements for our estimation error bounds to hold. 

Nevertheless, this introduces a seemingly large gap between sparse linear regression results where the sample complexity grows linearly with sparsity and our corresponding Poisson model. It is unclear whether this gap is due to matrix construction or our proof bounding technique. To bridge this gap we consider a different matrix construction and a variant of our optimization problem\footnote{We thank the anonymous reviewer for suggesting this scheme.}. 

In particular, we construct our matrix based on another elementary matrix ${\mathbf A}_g$. The components of ${\mathbf A}_g$ are i.i.d. instantiations of a bounded sub-gaussian random variable, whose minimum and maximum values are denoted as $-a_{\wedge} < 0$ and $a_{\vee} > 0$, respectively. The main reason for this construction is that we can now obtain designs with linear scaling, namely, for $n = \mathcal{O}(k \log p)$ it is possible to construct matrices with a constant factor $\gamma$ bound \cite{ShuZh09} in Eq.~\ref{REC}. We emphasize that it is important here for  ${a}_{i,j}$ to take both positive and negative values for this to hold.

We now construct matrices ${\mathbf A}$ based on ${\mathbf A}_g$, namely,  
\begin{equation} \label{AgC}
{\mathbf A} = {\mathbf A}_g + a_{\wedge} {\mathbf 1}_{n \times p}
\end{equation}
where ${\mathbf 1}_{n \times p}$ denotes an all one matrix of size $n \times p$. Finally, we consider a variant of our $\ell_1$ constrained problem:
\begin{equation} \label{eq:opt2}
\widehat{\mathbf w} := \argmin_{\sum_{i} {w}_i = s, ~\forall i ~:~ {w}_i \geq 0} Q({\mathbf w}) 
\end{equation}
The main difference between Eq.~\ref{eq:opt1} and Eq.~\ref{eq:opt2} is that the constraint $\| {\mathbf w} \|_1 \leq s$ is replaced by $ \| {\mathbf w} \|_1 = \sum_i w_i = s$. While this equality constraint is somewhat more restrictive it allows us the flexibility of utilizing matrices with negative components. Furthermore, it bridges the gap between the sparse linear regression setting and this problem as shown in the following corollary.
\begin{corollary}
Let $n \geq C k \log p$, $a_{\max} := \max(a_{\vee}, a_{\wedge})$, and ${\mathbf A}$ be constructed according to Eq. \eqref{AgC}, then the estimator $\widehat{\mathbf w}$, of Eq.~\ref{eq:opt2} satisfies the following bound:
$$ \|  \widehat{\mathbf w}  - {\mathbf w}^\star \|_2 \leq c_0 \lambda_{\max} a^2_{\max}  \sqrt{\frac{k \log\left({2}/{\zeta}\right)}{\overline{\lambda}_h n}} \left(  3 +  \log \frac{\lambda_{\max}}{\lambda_{\min}}  \right),$$
with probability of at least $1 - \zeta - c_1 \exp(-c_2 n)$, where $C$, $c_0$, $c_1$, and $c_2$ are positive universal constants, and $\zeta$ satisfies the conditions of Theorem 1.
\end{corollary}


\subsection{Minimax Mean Squared Error Lower Bound}

In this section, we derive a lower bound for the mean-squared error and show that the upper bounds obtained in the previous section match our lower bound. Our proof technique is based on a generalization of Fano's inequality~\cite{Bin97}. In particular, we show that no estimation algorithm can have a worst case mean squared error (MSE) less than $\mathcal{O}(\sqrt{sk/n})$:
\begin{theorem}
	Let $\eta$ and $a_{\min}$ be the maximum eigenvalue of ${\mathbf A}/\sqrt{n}$, and minimum element of ${\mathbf A}$, respectively. Furthermore, let $\mathcal{G} := \{ {\mathbf w} : \| {\mathbf w} \|_0 \leq k, ~ \| {\mathbf w} \|_1 \leq s, ~ \forall i ~:~ w_i \geq 0 \}$ be the set of feasible parameter values. If $n \geq C a_{\min} (k - 1)^2 / (s \eta)$ and $k \geq 9$, the minimax error rate is bounded as follows:
	$$ \inf_{\widehat{\mathbf w}} \sup_{{\mathbf w}^\star \in \mathcal{G}} \E(\| \widehat{\mathbf w} - {\mathbf w}^\star \|_2) \geq C^\prime \sqrt{\frac{a_{\min} s k}{n \eta}},$$ 
	where $\widehat{\mathbf w}$ is an arbitrary estimation algorithm, and $C$ and $C^\prime$ are two constants. 
\end{theorem}

\noindent
{\bf Matching Lower Bound}: The error bound in Theorem 2 matches (order-wise) the upper bound in Theorem 1 as long as $\lambda_{\min} = \Omega(s)$. This establishes the fact that $\ell_1$ constrained ML decoder is minimax optimal in the regime where the Poisson rates vary within a constant factor.  \\

\noindent	
{\bf Brief Proof Outline}: Let $\mathcal{G}_M$ be a subset of $\mathcal{G}$ of size $M$. We consider $\mathcal{G}_M$ as a set of hypotheses. It follows from Fano's bound \cite{Bin97} that the worst case MSE of any estimator is lower bounded by $0.5 P_e \delta$, where $P_e$ is the hypothesis test error probability and $\delta$ is the minimum $\ell_2$ distance of any two members of $\mathcal{G}_M$. We are now left to construct an appropriate hypothesis subset. We utilize the Gilbert-Varshamov bound~\cite{Gun11} in this context to construct a subset such that $\delta$ is lower bounded by $\Omega(\sqrt{sk/n})$ for $M = \Omega(\exp(k/8))$. Furthermore, $P_e$ can be lower bounded through the Fano bound by ensuring that KL-divergence of any two observation distributions with parameters in $\mathcal{G}_M$ is no more than $\mathcal{O}(k)$. The detailed proof is provided in the Appendix, Sec. \ref{Thm2Pr}.

\begin{figure}
\begin{minipage}[htb]{1\linewidth}
\centering
\centerline{\includegraphics[scale=0.3, angle=0]{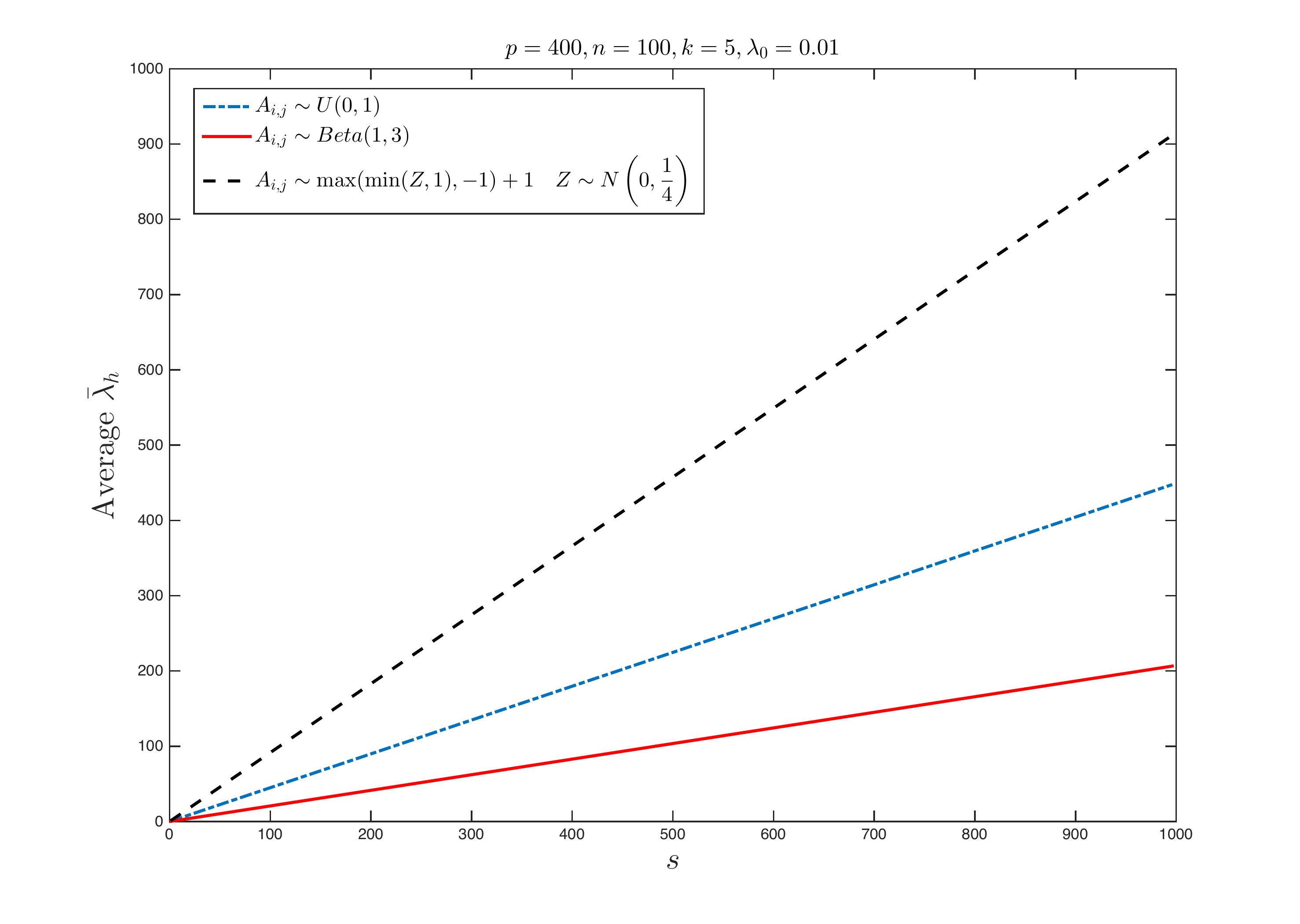}}
\caption{Linear scaling of average $\bar \lambda_h$ w.r.t signal intensity $s$ for three different random constructions of $\mathbf A$.}
\label{lambdahFig}
\end{minipage}
\end{figure}

\section{Numerical Results}
\label{NR}
\subsection{Poisson Rates Harmonic Mean} 
In this section, we consider three different random matrice constructions and verify whether $\bar \lambda_h$ scales with $s$. This linear scaling together with the bounds derived in Theorem 1 and Corollary 1 imply that the error matches the minimax rate $\mathcal{O}(\sqrt{ks/n})$ up to a $\log s$ factor. Specifically, we consider i.i.d. copies of three different distributions to construct $\mathbf{A}$. Namely, uniform distribution on $[ 0, 1 ]$, Beta distribution with $\alpha = 1$ and $\beta = 3$, and a random variable $V$ defined as
\begin{equation} \label{AltDist} 
V = \begin{cases} 0 ; & \text{if } Z < -1 \\ Z + 1 ; & \text{if } -1 \leq Z \leq 1 \\ 2 ; & \text{if } Z > 1 \end{cases}
\end{equation}
where $Z$ is a normally distributed random variable with $\mu = 0$, $\sigma^2 = 0.25$. We take the average of $\bar \lambda_h$ over $50$ Monte Carlo runs. In each run, ${\mathbf A}$ and ${\mathbf w}^\star$ are selected at random. Elements of ${\mathbf w}^\star$ on the support are selected from the uniform distribution on $[ 0, 1]$, and the support is equiprobable across all possible supports. The average value of harmonic mean is plotted against $s$ in Fig. \ref{lambdahFig}. This verifies the linear scaling of $\bar \lambda_h$ w.r.t $s$.

\subsection{Tightness of the Error Bounds} 
\label{TRV}
Theorem~2 characterizes the minimax optimality of $\ell_1$ constrained ML decoder in a certain regime where rates vary within a constant factor. On the other hand the upper bound in Theorem~1 and Corollary~1 scales as ${s \over \sqrt{\bar \lambda_h n}}$. As noted in Fig.~\ref{lambdahFig} the Harmonic mean, $\bar \lambda_h$, generally scales linearly with $s$. Consequently, the upper bound matches our worst-case lower bound. Nevertheless, it is still possible that our upper bounds for ML constrained decoder is conservative for the average case. In this experiment we attempt to show that the upper bounds are indeed tight for the average case as well.


\begin{figure*}
\begin{minipage}[b]{1\linewidth}
\centering
\centerline{\includegraphics[scale=0.25, angle=0]{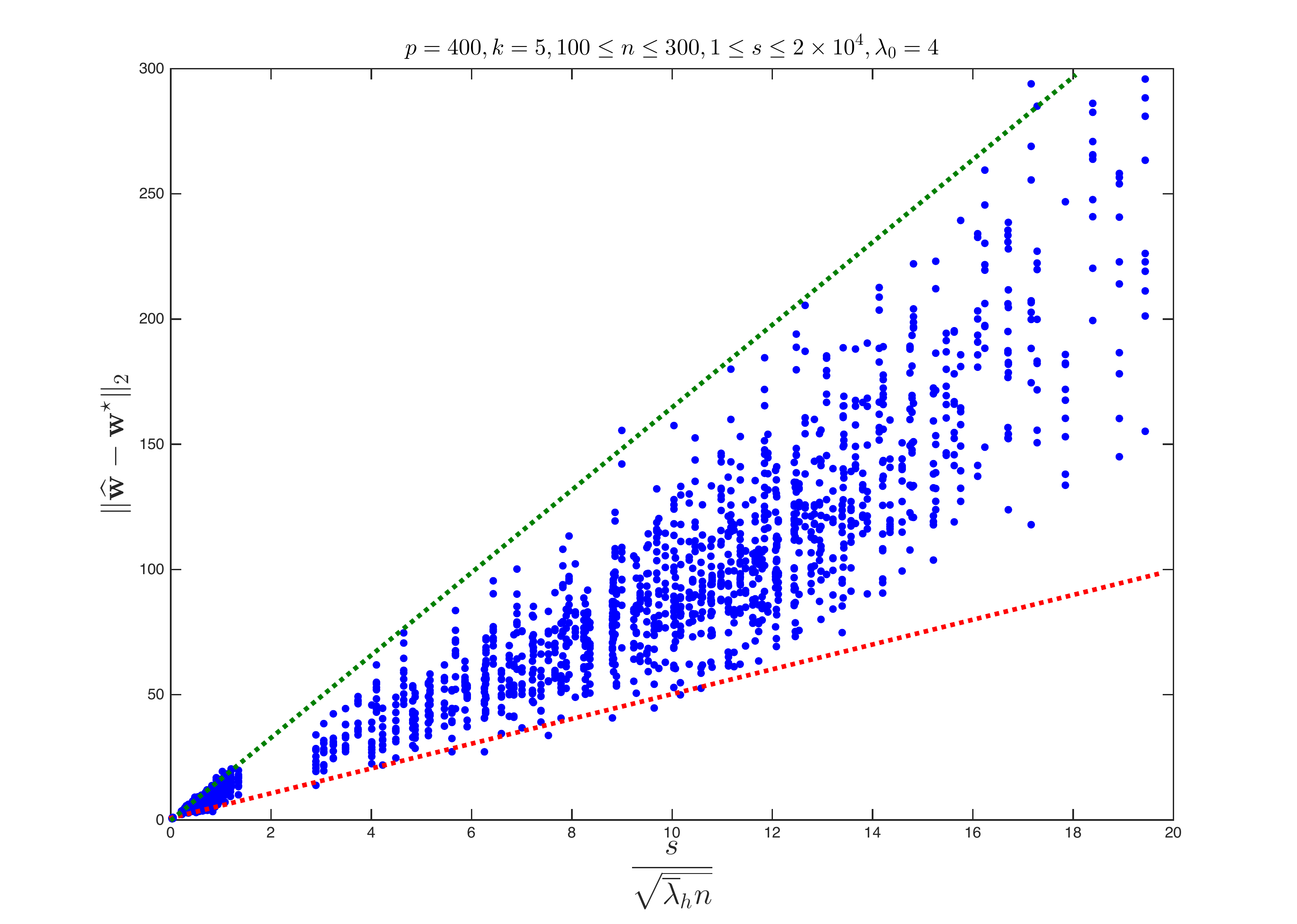}
\includegraphics[scale=0.25, angle=0]{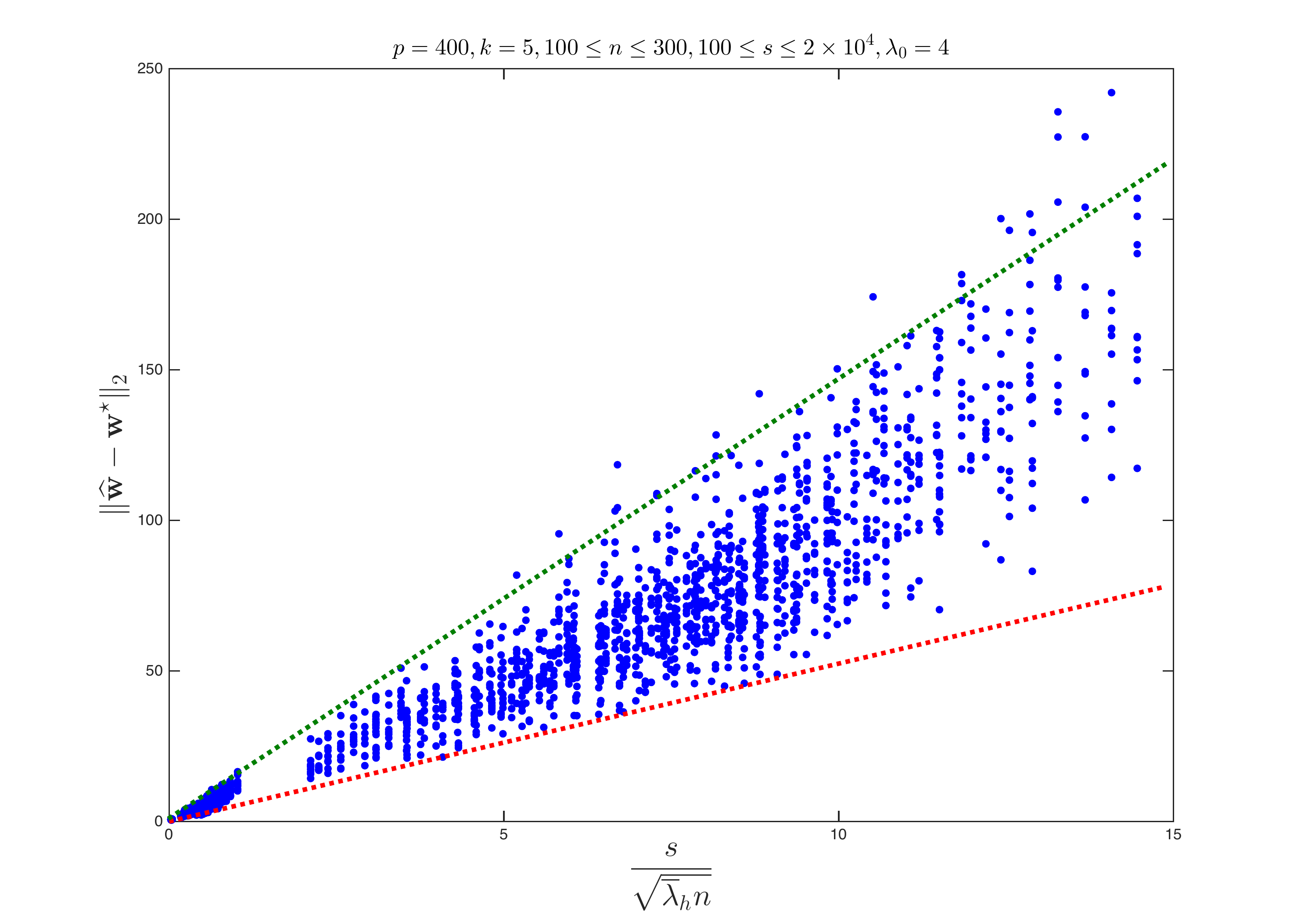}}
\caption{$\ell_2$ error vs. scaled derived error upper bound. This demonstrates that actual error lies within two linear scalings of the derived bound {\it with significant probability}, which in turn shows that derived bound is tight. Entries of ${\mathbf A}$ are drawn from the uniform distribution over $[0, 1]$ (left) and an alternative distribution defined in Eq. \eqref{AltDist} (right).}
\label{Fig1}
\end{minipage}
\end{figure*}

These reasons motivate us to perform experiments with synthetic data to verify the tightness of the derived bound in Theorem 1. Elements of ${\mathbf A}$ are drawn from the uniform distribution over the unit interval. We also consider realizations where the elements are i.i.d. instantiations of a suitably truncated Gaussian random variable defined in Eq. \eqref{AltDist}.

We set the dimensionality $p = 400$, sparsity level $k = 5$, sample size $100 \leq n \leq 300$, signal strength $1 \leq s \leq 2 \times 10^4$, and the base rate $\lambda_0 = 4$. We fix the direction of ${\mathbf w}^\star$ throughout the experiment and change the scale parameter $s$. We also fix ${\mathbf A}$ throughout all the experiments. Therefore, the uncertainty in drawing Poisson distributed data is the only uncertainty across all the experiments. The experiment is repeated (for the same values of ${\mathbf w}^\star$ and ${\mathbf A}$) for $10$ times. Notice that for the truncated normal distribution, we set some components of design matrix to zero. 

Recorded points in all the experiments are then plotted. We use the interior-point optimization algorithm implemented in MATLAB with the termination tolerance value of $10^{-10}$ on the objective function. The termination tolerance on the variables is also set to $10^{-10}$. The gradient of the objective function is also given to the optimization algorithm initially. 

In each run of the experiment, the $\ell_2$ norm of the error is recorded. The $\ell_2$ error is plotted against $s/\sqrt{n \overline{\lambda}_h}$ in Fig. \ref{Fig1}. The hypothetical dashed lines passing through the origin shows that the error scales as $s/\sqrt{n \overline{\lambda}_h}$ with high probability.  This suggests that our upper bounds are tight even for the average case. 
The estimation error is also a function of sparsity level $k$. We next experiment with varying levels of sparsity, $k$, while fixing the scale $s$. These results are consistent with our bounds. This is not surprising because both our lower and upper bounds suggest linear scaling of squared error with $\sqrt{k}$. This has been illustrated in Fig. \ref{FigKEff} for the case that $A_{i,j} \sim U(0, 1)$. We do not consider experiments in which dimensionality $p$ is changed, simply because the derived bounds do not explicitly rely on $p$ as long as RE condition is satisfied. This happens for random constructions, when $n$ is large compared to $k$ and $\log p$, i.e. $n = \mathcal{O}(k \log p)$.

\begin{figure}
\begin{minipage}[b]{1\linewidth}
\centering
\centerline{\includegraphics[scale=0.35, angle=0]{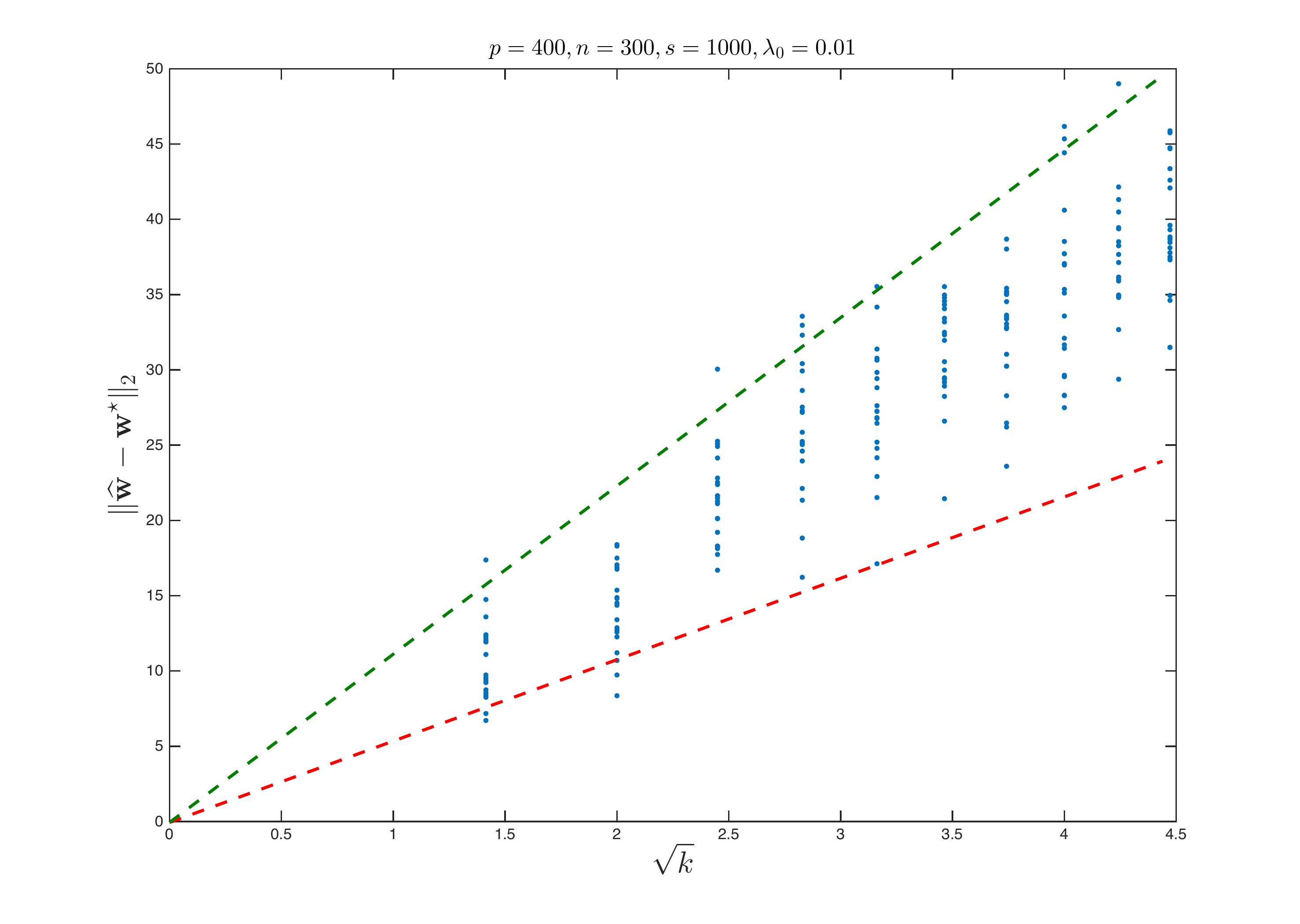}}
\caption{$\ell_2$ error changes with $\sqrt{k}$ when $A_{i,j} \sim U(0, 1)$. This is consistent with the upper and lower bounds in Theorems 1 and 2.}
\label{FigKEff}
\end{minipage}
\end{figure}


\subsection{Gaussian vs. Poisson Model}
In this section, we compare the $\ell_2$ norm error of he constrained maximum likelihood decoder, by using a Gaussian likelihood function even though the observations are generated using a Poisson model. We then compare the resulting estimation error with the estimation error obtained using a Poisson likelihood function. Elements of ${\mathbf A}$ are generated according to the Beta distribution with $\alpha = 1$ and $\beta = 3$, and are sampled independently during 500 Monte Carlo runs. During each run a weight vector ${\mathbf w}^\star \in \mathbb{R}^{400}_+$ is chosen at random, and $n = 100$ Poisson observations are made based on the model described in Sec. \ref{PS}. We set $\lambda_0 = 0.01$ and change $s$ from $1$ to $1000$ or alternatively fix $s = 200$ and change $\lambda_0$ from $0$ to $10$. We plot the average of ratio of $\ell_2$ error in the two cases in Fig. \ref{Figx}. This figure shows that Poisson model outperforms the Gaussian one especially when both $s$ is large and $\lambda_0$ is small. This could be partly explained by the fact that under this condition, the variance of observations is small relative to the mean. Hence, if the background noise is small, each observation would convey more information about ${\mathbf w}^\star$, making performance of the true model significantly better than the conventional Gaussian one. 

\begin{figure*}
\begin{minipage}[ht]{1\linewidth}
\centering
\centerline{\includegraphics[scale=0.3, angle=0]{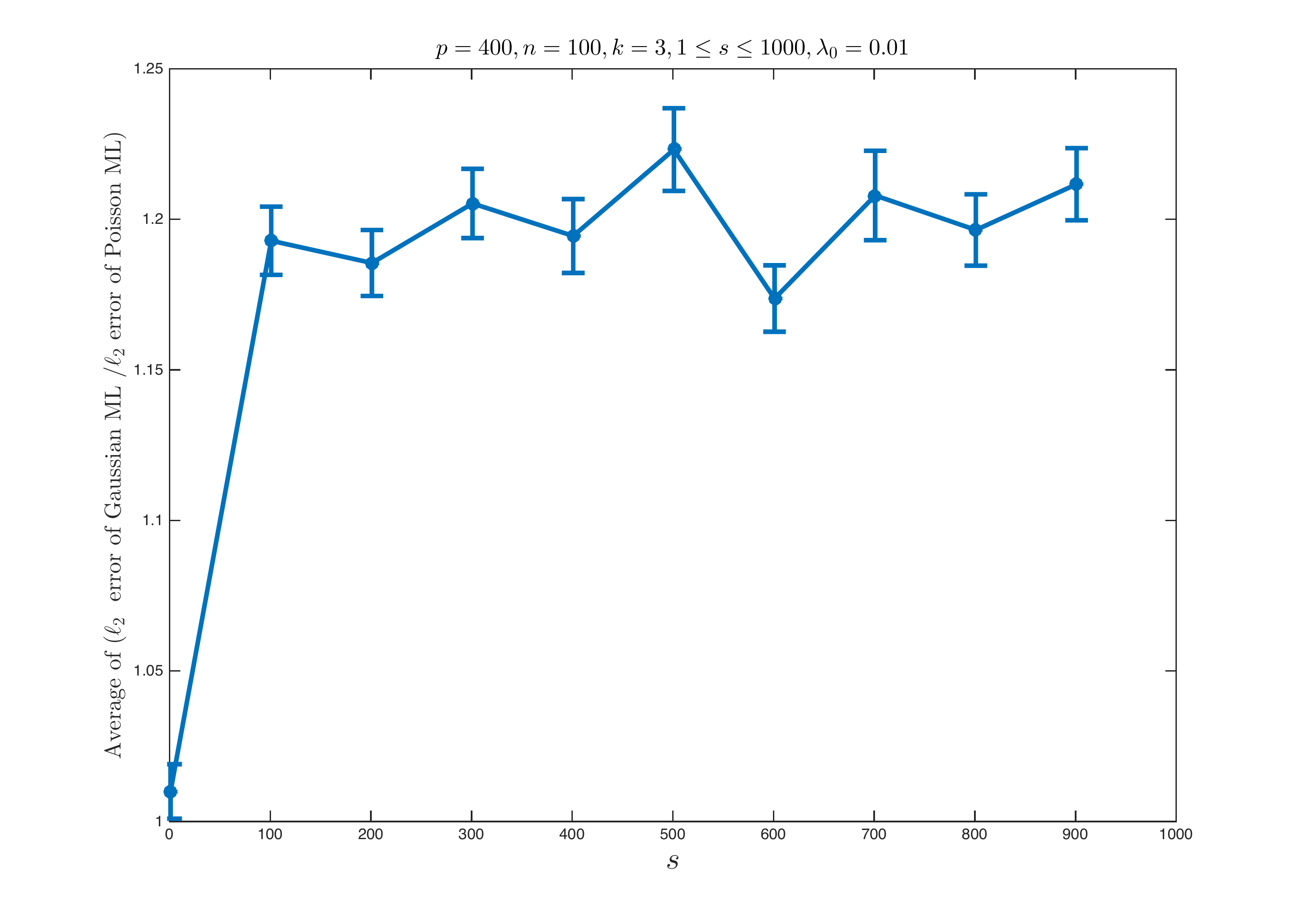}
\includegraphics[scale=0.3, angle=0]{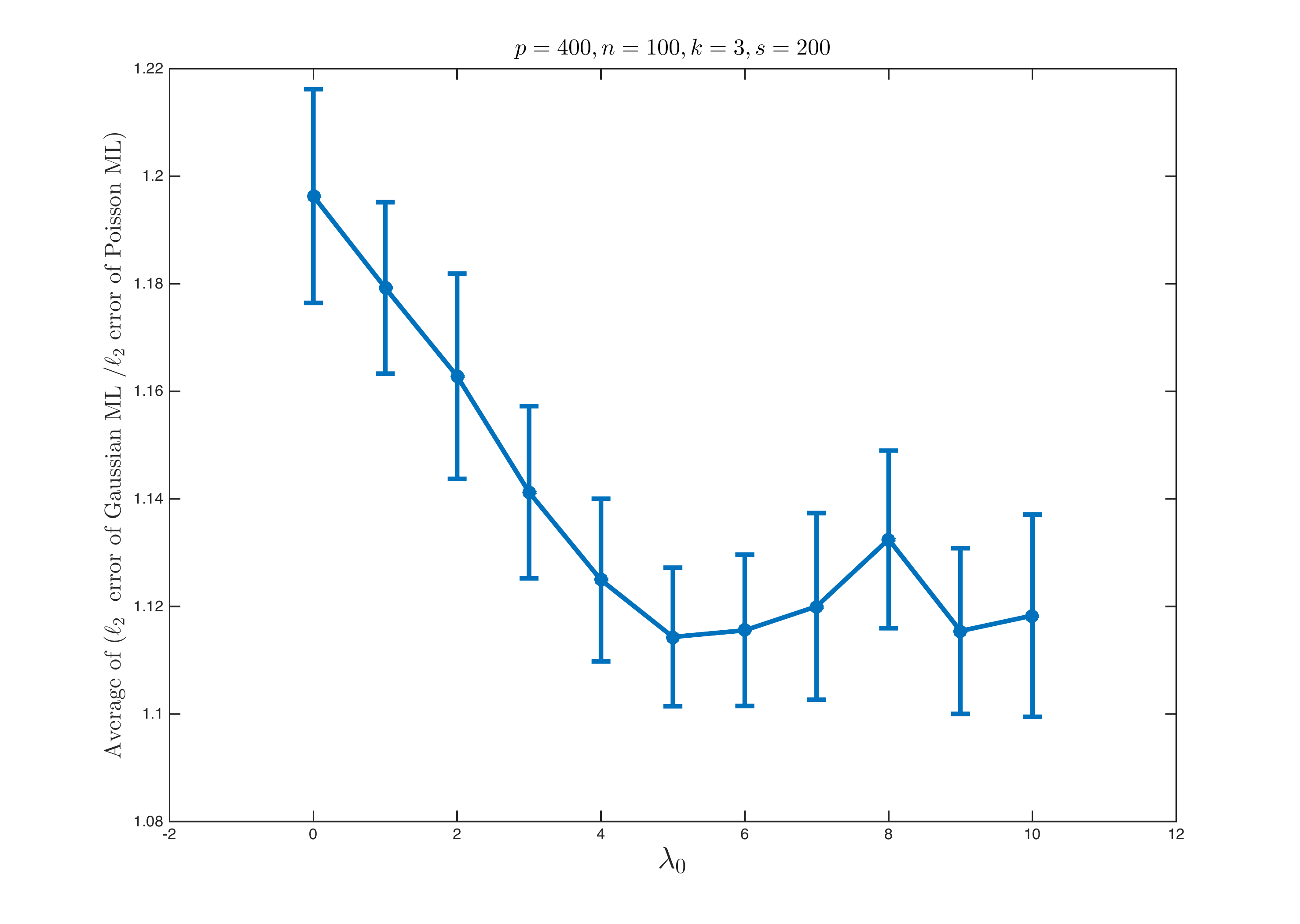}}
\caption{$\ell_2$ errors of Maximum Likelihood estimation assuming a Gaussian observation model vs. assuming the true Poisson distribution. The ratio of errors in two cases across different signal amplitudes (left) shows that the true Poisson model performs constantly better than the Gaussian one when signal amplitude is large. The same plot across various $\lambda_0$ values (right) shows that the difference between two models is significant when the background noise is small. These two plots together suggest that under high SNR values there is a significant difference in the performances of the two models.}
\label{Figx}
\end{minipage}\end{figure*}

\subsection{Rescaled LASSO vs. Regularized ML}
Parameter estimation based on LASSO for the Poisson setting has been studied in \cite{cite5}. The idea is to view the problem as an additive noise problem, where noise belongs to an exponential family of distributions. Alternatively, in \cite{cite5} the problem is viewed as an additive Gaussian noise problem with noise variance being equal to its mean to mimic ``Poisson like'' behavior. This results in a rescaled version of LASSO, which is then used to estimate model parameters. This amounts to scaling the loss function associated with each observation by the corresponding mean value (or equivalently the variance). 

This approach motivates us to compare the regularized ML method against rescaled LASSO for poisson distributed data. This will highlight the essence of using regularized ML instead of rescaled LASSO for our setting. In this section, we will demonstrate that our regularized ML outperforms rescaled LASSO in several regimes including low SNR, high dimensions, and moderate to low sparsity levels. We use probability of successful support recovery as a measure of performance. We use this measure because of the scale invariance property of this measure. It avoids the issue of signal strength $s$ in impacting the performance metric and obviates the need to normalize the measure with respect to $s$, as we did in the last section.

To compare the performance of regularized ML and rescaled LASSO, we first generate a random sensing matrix ${\mathbf A} \in \mathbb{R}^{n\times p}$ where each element $a_{i,j}$ is an independent truncated Gaussian random variable. We then generate a random ${\mathbf w}^\star$. To recover ${\mathbf w}^\star$, we draw $n$ Poisson distributed samples with coefficients specified in ${\mathbf A}$ as:
$$y_i \sim \Poisson(\lambda_{0}+{\mathbf a}_i^{\T}{\mathbf w}^\star)$$
We first solve the non linear optimization where ${\mathbf w}$ is constrained to satisfy $\forall i ~:~ {w}_i \geq 0, \sum_{i} {w}_i \leq s$. 
$$\widehat{\mathbf w}_{ML}=\argmin_{\forall i ~:~ {w}_i \geq 0, \sum_{i} {w}_i \leq s}-\frac{1}{n}\sum_{i=1}^ny_i\log(\lambda_{0}+{\mathbf a}_i^{\T}{\mathbf w})-{\mathbf a}_i^{\T}{\mathbf w}$$
For the purpose of comparison, we also compute the rescaled LASSO estimator:
$$\widehat{\mathbf w}_{LS}=\argmin_{\forall i ~:~ {w}_i \geq 0, \sum_{i} {w}_i \leq s}\frac{1}{n}\sum_{i=1}^n\frac{(y_i-\lambda_{0}-{\mathbf a}_i^{\T}{\mathbf w})^2}{\lambda_{0}+{\mathbf a}_i^{\T}{\mathbf w}}$$

We then threshold the solution by zeroing out components of $\widehat{\mathbf w}_{ML}$ and $\widehat{\mathbf w}_{LS}$ below a pre-defined small threshold $t$. 
We average the estimation performance over 100 Monte Carlo runs. The performance of the two methods are compared in Fig. \ref{tr} and Fig. \ref{k}. The results are compared for different number of observations $n$, and different sparsity levels $k$, respectively.

In Fig. \ref{tr}, we compare the performance of regularized ML estimation with that of rescaled LASSO as a function of $n$. We fix $\lambda_0=100$, $p=400$, $t=10^{-4}$, $s = 1$ and $k=40$. At each iteration, we estimate ${\mathbf w}^\star$ based on $n$ observations where $n$ varies from $2$ to $400$. We compare the performance of the two approaches based on probability of successful recovery of the support set. This error is 0 if the thresholded support set of the estimation is equal to that of the ground truth and 1 otherwise. We average this measure over 100 samples of ${\mathbf w}^\star$ for a fixed ${\mathbf A}$.
\begin{figure}
\begin{minipage}[b]{1\linewidth}
  \centering
  \centerline{\includegraphics[width= 9cm]{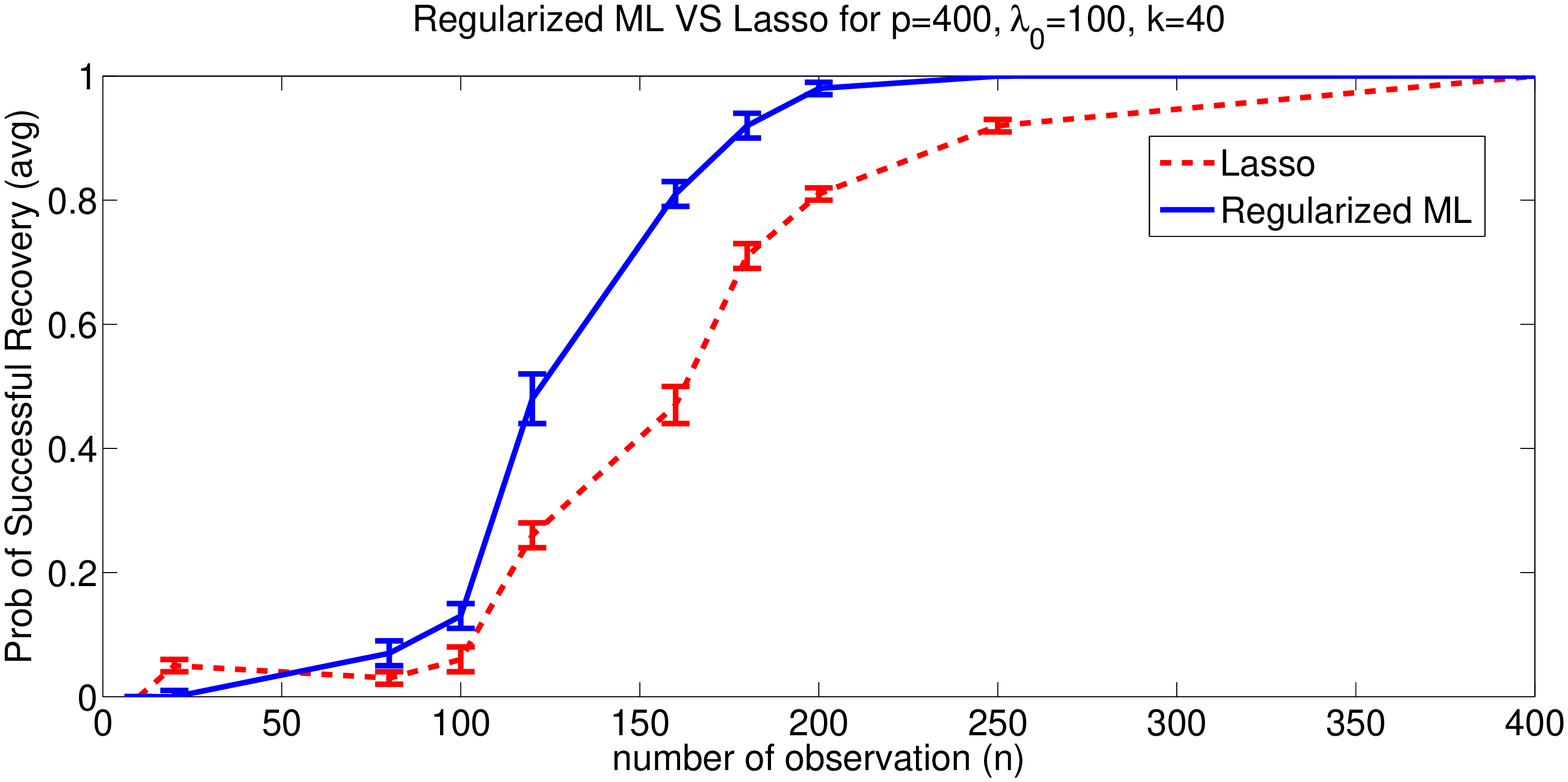}}
\caption{Probability of successful support recovery as a function of $n$ for $p=400$, $\lambda_0=100$, $k=40$, $t=10^{-4}$, and $m=100$ Monte Carlo loops. This figure illustrates twice faster convergence of probability of success with respect to number of observations for Regularized ML in comparison to Rescaled LASSO.
}
\label{tr}
\end{minipage}
\end{figure}

In Fig. \ref{k}, we compare the performance of regularized ML estimation with that of rescaled LASSO for different sparsity levels, $k$. This time, we fix $\lambda_0=100$, $p=200$, $s = 1$, and $n=100$. For each $k$, we generate 100 samples of $k$-sparse ${\mathbf w}^\star$'s and recover them from $n$ observations. Since $\| {\mathbf w}^\star \|_1=1$ for all values of $k$, we threshold each element of $\widehat{\mathbf w}$ by $t = \frac{0.01}{k}$, to obtain their sparse support set. We measure the performance of the two estimators based on average probability of successful recovery of the thresholded support set for each value of $k$. 
Notice that the error bars in Fig. \ref{tr} and Fig. \ref{k} indicate that the difference between the methods is indeed statistically significant.
\begin{figure}
\begin{minipage}[b]{1\linewidth}
  \centering
  \centerline{\includegraphics[width= 9cm]{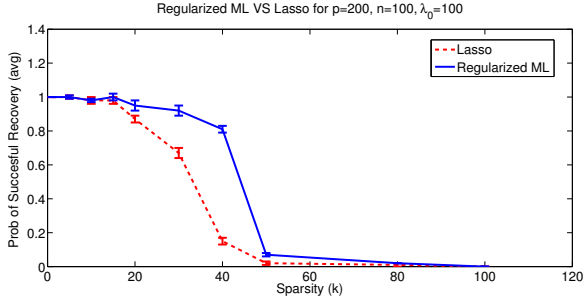}}
\caption{Probability of successful support recovery as a function of $k$ for $p=200$, $\lambda_0=100$, $n=100$, and $m=100$ Monte Carlo loops. As this figure suggests, probability of successfull recovery drops faster when the sparsity level increases for Rescaled LASSO in comparison to Regularized ML. This shows robustness of ML approach to model parameters.
}
\label{k}
\end{minipage}
\end{figure}

In Fig. \ref{ROC}, we compare the result of regularized ML estimation with that of rescaled LASSO in terms of the ROC curves. In our ROC curve, we plot the average number of true detections against the average number of false alarms. True detections are indices that are common in the thresholded estimated support set and that of the Ground Truth, whereas, false alarms are the indices in the thresholded estimated support set that are not included in the support set of the Ground Truth. This time, we fix $\lambda_0=100$, $p=200$, $n=100$, $s = 1$, and $k=20$. We fix a sensing matrix ${\mathbf A}$, and generate 100 random ${\mathbf w}^\star$'s. By applying different thresholds $t=\frac{1}{k}$ to $t=\frac{0.001}{k}$ we obtain the different points in the ROC plot. We average Probability of Detection (PD) and Probability of False alarm (PF) over $100$ Monte Carlo runs.
\begin{figure}
\begin{minipage}[b]{1\linewidth}
  \centering
  \centerline{\includegraphics[width= 9cm]{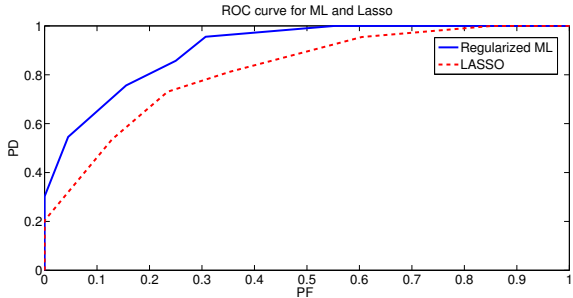}}
\caption{ROC curve for Regularized ML and Rescaled LASSO for  $n=100$, $p=200$, $k=20$, $\lambda=100$, and $m=100$ Monte Carlo loops. This figure illustrates the superiority of Regularized ML in comparison to Rescaled LASSO for parameter estimation under Poisson models.
}
\label{ROC}
\end{minipage}
\end{figure}



\subsection{Explosive Identification}

In this experiment, we first measure the light intensities of different fluorophores before and after separate exposures to a unit weight of different explosives. The intensities are measured by counting the number of photons received at each photo-sensor. Each explosive $j$ has a unique quenching effect in the fluorescence property of each fluorophore $i$, which we denote by $a_{i,j}$. In the experimental setting, $\lambda_i$ is the before exposure intensity for fluorophore $i$ and is estimated by averaging the before exposure photon counts from multiple experiments. Therefore, $\lambda_i$'s can be assumed to be known. We model the after exposure intensity of $j$-th explosive in sensor $i$ as :
$y_i \sim \Poisson(\lambda_i(1-a_{i,j}))$

In the next step, fluorophores are exposed to an unknown mixture of these explosives. The goal is to recover which and how much of each explosive is contained in that mixture.

The physics of the problem suggests that when the fluorophore is exposed to a mixture of explosives, the quenching effects are additive in the regime where the mixture weights are small \cite{cite6}. Therefore, our observations are best modeled by a Poisson distribution with additive rate model for each fluorophore:
$$ y_i \sim \Poisson\left(\lambda_i(1- {\mathbf a}_i^\top {\mathbf w}^\star)\right)$$ 
where $w^\star_j$ is the amount of the explosive $j$ in the mixture. We solve this problem through Regularized ML and Rescaled LASSO and compare the results.

In this problem, matrix ${\mathbf A}$, the responses of $n=8$ fluorophores to $p=12$ basic explosives is given. Based on this given data and our additive model for mixtures, we generated 10 mixtures by combining up to 3 random explosives. We used Regularized ML and Rescaled LASSO to identify these mixtures through their effect on fluorescence property of our fluorophores. The result is shown in the form of a $10\times 12$ grid in Fig. \ref{FlComp}. In this grid, rows are different mixtures and columns are different explosives. Dark squares indicate the absence (or negligible contribution), whereas lighter squares indicate higher amount of the corresponding explosive in the associated mixture.
\begin{figure}[htb]
\centering
\centerline{\includegraphics[width= 9cm]{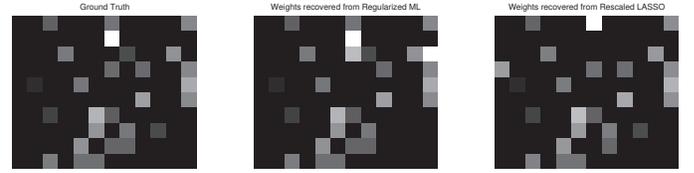}}
\caption{Sparse recovery results for $k\leq 3$. From left to right: Ground Truth, ${\mathbf w}^\star$, ML estimate of ${\mathbf w}^\star$, Rescaled LASSO estimate of ${\mathbf w}^\star$. Columns: Basic explosives.
Rows: Synthesized mixtures with $k$ basic explosives. Non-black squares at each row show the explosives that the corresponding mixture is composed of. Lighter colors show larger amounts.
}
\label{FlComp}
\end{figure}


\subsection{Internet Marketing Application}
This application presents a different scenario where the ultimate objective is prediction and inference. These experiments are useful because they further our understanding of Poisson models in terms of the why/how and in what contexts Poisson models are more meaningful in comparison to Gaussian models for Poisson distributed data. While the goals here are outside the purview of the theory we have developed, it nevertheless serves the purpose of understanding sparse high-dimension Poisson models in the context of not only estimation but also for prediction. 

We discuss the eMarketing application in this section, where we will illustrate Poisson model is more meaningful than the conventional Gaussian noise model. In this application, our goal is to identify the most effective advertisement websites that result in higher website traffic in the clothing market. Our assumption is that the website traffic is generated as a superposition of the traffic generated from current customers and the traffic from advertisement through backward links (links in advertisement websites that are linked to these business websites). In general, big business websites typically buy a total of 1000-1500 backward links from a number of advertisement websites. However, the hypothesis is that only a few of these advertisement websites are efficiently directing costumers. Our goal is to identify those dominant advertisement websites.
 
We model the number of daily visits, $y_i$, by:
$$y_i \sim \Poisson(\lambda_{i,0}+{\mathbf a}_i^{\T}{\mathbf w}^\star)$$
where $\lambda_{0,i}$ models the current customers who visit the site directly and is obtained through online statistics of the website. Specifically, a long run average of traffic which is not referred by any advertisement website gives a reasonable estimate and so we can assume that $\lambda_{0,i}$'s are actually known. This traffic could be logged and acquired through online statistics of the business websites. 

Moreover, $a_{i,j}$ is the number of backwards link for the website $i$ in the advertisement website $j$. Our model assumes that each of the backward links brings independent traffic to the website. Therefore, we used the Poisson distributed random variable described earlier to model the number of visits to a business website.

Our observations are the daily online visits to 50 top clothing brands. From the information provided in alexia.com, we chose the top 150 advertisement websites for these brands along with the number of backward links for each website. Our goal is to recover the weight vector ${\mathbf w}^\star$, where $w^\star_j$ is a measure of dominance for advertisement website $j$ in clothing market. We recover ${\mathbf w}^\star$ via regularized ML and rescaled LASSO (weights smaller than 0.01 are theresholded to 0). The result is provided in Table 1. In this table, we illustrate the corresponding score for each popular website based on their dominance in advertising for clothing brands.

\begin{table}[h]
\caption{Top backwardlist websites for clothing brands using Regularized ML and Rescaled LASSO}
\begin{center}
\begin{tabular}{|l|c|c|c|c|c|c|c|c|c|}
               \hline
               Backward link & ML estimated weight \\
               \hline
               Amazon & 0.32\\
               Twitter & 0.21\\
               Pinterest & 0.17 \\              
               Google &  0.15\\
               Blogger &  0.06\\
               Bing & 0.05\\
               douban & 0.01\\
               tumblr & 0.01 \\              
               \hline
               \hline
               Backward link & Lasso estimated weight \\
               \hline
               Amazon & 0.35\\
               Pinterest & 0.17 \\   
               Twitter & 0.16   \\        
               Google &  0.16\\
               Bing & 0.13\\            
               \hline

\end{tabular}
\end{center}
\end{table}

To compare the result of the two approaches mentioned above, we use the Bayes factor \cite{Kass} and predictive held-out log likelihood comparison mentioned in \cite{Blei}. It should be mentioned that these tests are interpretable only when the number of parameters are comparable in the hypothesis models. In our problem, in fact, the two models have equal number of parameters.

Given a set of observed data $y_1, \hdots, y_n$, and a model selection problem in which we have to choose between two models, Bayesian inference compares the plausibility of the two different models $M_1$ and $M_2$ through a likelihood test:

$$\frac{\Pr\left(y_1, \hdots, y_n|M_1\right)}{\Pr\left(y_1, \hdots, y_n|M_2\right)}\lessgtr 1$$ 

When the parameters of models $M_1$ and $M_2$ are not known a priori, in Bayes factor test, we estimate them from $y_1, \hdots, y_n$ and then use those estimations in computing the likelihood ratio. On the other hand, in predictive held-out log likelihood comparison, we divide the data into two groups. We estimate the model parameters for $M_1$, and $M_2$ using the first group of data, and we compare the likelihoods for the second part of data given $M_1$ and $M_2$ specified by the first group.

Since Poisson is a PMF distribution on integers, to compare the two models using Bayes factor, we need to superimpose the Gaussian distribution on a histogram defined on integer valued $y_i$'s. For a Gaussian distribution characterized by $\mathcal{N}(\mu,\sigma)$, the value of the histogram at each integer valued $y$ is computed as:
\begin{equation}
\label{hist}
\hist(y)=\frac{1}{Q(\frac{\mu}{\sigma})}\times \left(Q\left(\frac{y-\mu}{\sigma}\right)-Q\left(\frac{y+1-\mu}{\sigma}\right)\right)
\end{equation}
It is easy to show that this histogram corresponds to a valid PMF. We denote this PMF by $\overline{\mathcal{N}}(\mu,\sigma)$. 

After this conversion, the Bayes factor as a function of sparsity level, $k$, is calculated as:
{\fontsize{0.85em}{0.1cm} \selectfont
$$BF_{k}=\frac{\Pr\left(y_1,\hdots, y_n|y_i\sim \Poisson(\lambda_{0, i} +{\mathbf a}_i^{\T} \widehat{\mathbf w}^k_{ML})\right)}{\Pr\left(y_1, \hdots, y_n|y_i \sim\overline{\mathcal{N}}(\lambda_{0, i} + {\mathbf a}_i^{\T} \widehat{\mathbf w}_{LS}^k,\lambda_{0, i} + {\mathbf a}_i^{\T} \widehat{\mathbf w}^k_{LS})\right)}$$ }
where $\widehat{\mathbf w}^k_{ML}$ and $\widehat{\mathbf w}_{LS}^k$ are $k$ sparse thresholded approximations of $\widehat{\mathbf w}_{ML}$ and  $\widehat{\mathbf w}_{LS}$, respectively. The Bayes factor log curve as a function of sparsity is presented in Fig. \ref{cox}.
\begin{figure}
\begin{minipage}[b]{1\linewidth}
  \centering
  \centerline{\includegraphics[width= 9cm]{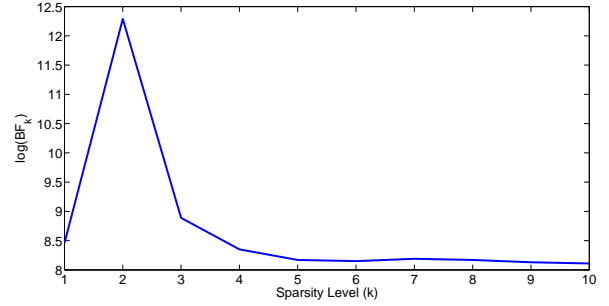}}
\caption{The Bayes Factor Ratio for regularized ML and rescaled LASSO. As we can see, higher Bayes Factor suggests that Poisson setting introduced earlier is a better model for this problem compared to the traditional mean squared loss.
}
\label{cox}
\end{minipage}
\end{figure}
To compute the predictive held-out log likelihood for each method, we first use 80\% of the data (40 training data) to calculate an estimation of the parameters, $\widetilde{\mathbf w}_{ML}$ and $\widetilde{\mathbf w}_{LS}$. We use $\widetilde{\mathbf w}_{ML}$ and $\widetilde{\mathbf w}_{LS}$ for each model to compute the log likelihood function for the remaining 20\% of data (10 test data):
$$\mathcal{L}_{ML}=\sum_{i=1}^{10}-\lambda_{0, i} -{\mathbf a}_i^{\T}\widetilde{\mathbf w}_{ML}+y_i\log(\lambda_{0, i}+{\mathbf a}_i^{\T}\widetilde{\mathbf w}_{ML})-\log(y_i!)$$
{\fontsize{0.85em}{0.1cm} \selectfont
\begin{align*}&\mathcal{L}_{LASSO}=\sum_{i=1}^{10}-\log\left(\!Q(\sqrt{\lambda_{0, i}+{\mathbf a}_i^{\T}\widetilde{\mathbf w}_{LS}})\right)+\\&\log\!\left(\!Q\left(\frac{y_i-\lambda_{0, i}-{\mathbf a}_i^{\T}\widetilde{\mathbf w}_{LS}}{\sqrt{\lambda_{0, i}+{\mathbf a}_i^{\T}\widetilde{\mathbf w}_{LS}}}\!\right)\!-\!Q\left(\frac{y_i+1-\lambda_{i, 0} -{\mathbf a}_i^{\T}\widetilde{\mathbf w}_{LS}}{\sqrt{\lambda_{0, i}+{\mathbf a}_i^{\T}\widetilde{\mathbf w}_{LS}}}\right)\right)
\end{align*}
}
where $Q$ stands for the tail probability of standard normal distribution here. Intuitively, the model that is closer to the ground truth results in higher log likelihood value. The log likelihood values for the two approaches are shown in Fig. \ref{cox1}. The large gap between the predictive log likelihood of the two models implies that Poisson is a better underlying model for this application.
\begin{figure}
\begin{minipage}[b]{1\linewidth}
  \centering
  \centerline{\includegraphics[width= 9cm]{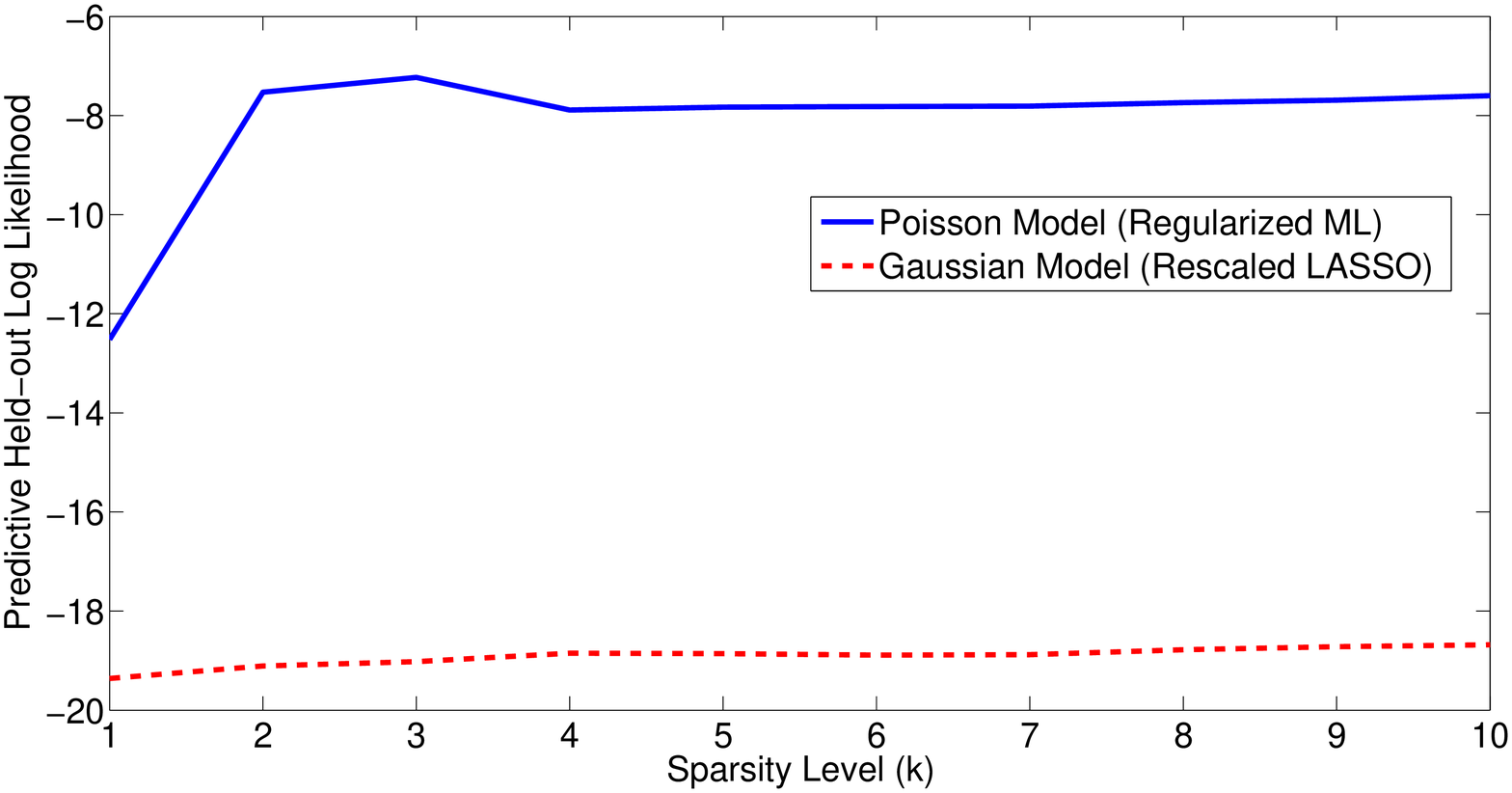}}
\caption{The Predictive Held-out Log Likelihood Comparison. The huge gap between the Predictive Held out Likelihood of Regularized ML and Rescaled LASSO implies that Poisson is the right model for this problem therefore, Regularized ML approach should be taken.
}
\label{cox1}
\end{minipage}
\end{figure}
\subsection{Dynamics of Online Marketing}
In the previous section, our results show that ML estimator and Poisson model outperforms LASSO approach for the problem of online marketing. Therefore, in this section, we apply ML method to estimate the weights, ${\mathbf w}^\star$, for the advertisement websites over time.
\begin{figure}
\begin{minipage}[b]{1\linewidth}
  \centering
  \centerline{\includegraphics[width = 9cm]{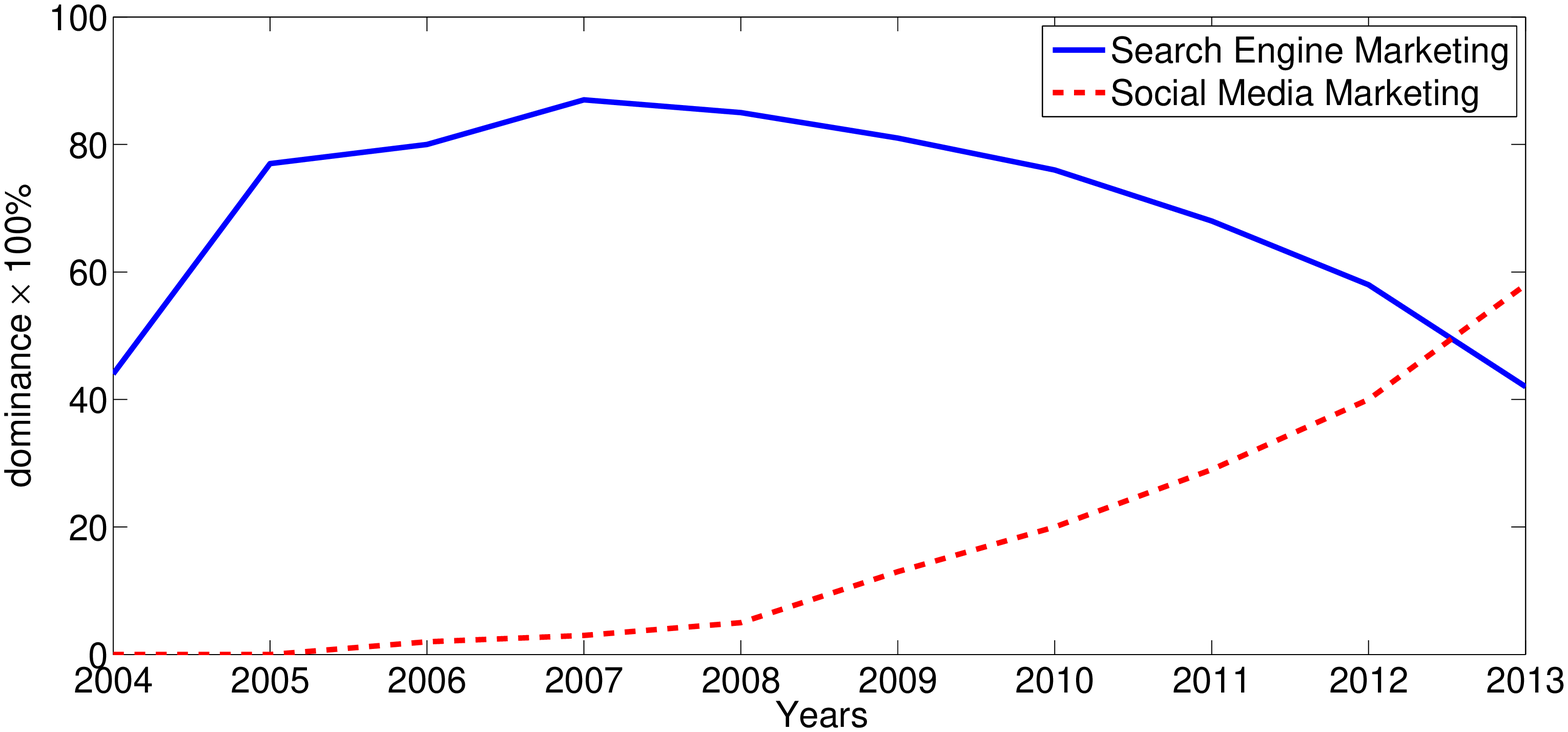}}
\caption{Dynamics of SEM and SMM over time for clothing market. This figure gives a quantitative comparison of the two most controversial methods of online marketing.
}
\label{cox2}
\end{minipage}
\end{figure}
\begin{table}[h]
\small{\caption{Top advertisement websites for clothing market in 2013 and 2008}
\begin{center}
\scalebox{0.7}{
\begin{tabular}{|l|c|c|c|c|c|c|c|c|c|}
               \hline
               Backward link & Estimated weights in 2013 &  Estimated weights in 2008\\
               \hline
               Twitter.com & 0.18 & 0.02\\
               Facebook.com & 0.18 & 0.02\\
               Pinterest.com & 0.14 & 0.00\\
               Amazon.com & 0.28 & 0.22\\              
               Google.com &  0.15 & 0.52\\
               Bing.com & 0.05 & 0.00\\
               Yahoo.com & 0.00 & 0.11\\
        
               \hline
\end{tabular}}
\end{center}}
\end{table}

A brief look at Table II shows how $w^\star_j$'s have changed dramatically over time. To study this change closely, we estimated $w^\star_j$'s for different advertisement websites from 2004 to 2013. We group the Social networks, such as facebook.com, twitter.com, pinterest.com, etc, together to study the effect of Social Media Marketing (SMM). We also group search engines, such as google.com, yahoo.com, bing.com, etc, together to represent Search Engine Marketing (SEM). We add the scores of the corresponding websites in each group. \hbox{Fig. \ref{cox2}} demonstrates the dynamics of SMM and SEM, the most controversial forms of online marketing, over time \cite{cite13}. Although SEM has been thought to be the most powerful media marketing tool, recent empirical studies show the growing influence of SMM during the last couple of years \cite{MF}.  The gigantic size of social media coupled with the relatively low cost per impression and the so called word of mouth have made SMM a powerful marketing tool. Our results confirm the significant influence of SMM relative to SEM since 2012.

\section{Conclusions}
In this paper we proposed a high dimensional sparse estimation problem where the observations are governed by heterogeneous Poisson rates motivated by applications such as Online Marketing and Explosives detection. 
Unlike least-squares linear regression setting, we showed that the scale of the parameter has a significant effect on sample complexity. In particular we derived upper bounds by analyzing the performance of an $\ell_1$ constrained decoder and develop matching lower bounds for worst-case mean-squared error. We presented a number of synthetic and real-world experimental results to not only demonstrate the importance of scale in Poisson problems but also the superior performance of $\ell_1$ constrained ML over rescaled LASSO for a number of problems and metrics that are of practical interest. 

\section{Appendix A} 
\label{AP}

\subsection{Definitions}
\begin{definition}
{\bf Strong Convexity} : A differentiable function $Q(.) : \mathbb{R}^p \rightarrow \mathbb{R}$ is called Strong Convex around a point ${\mathbf w}^\star$ over a set of perturbations $\mathbb{C}$ when the following holds for constants $\kappa$, and $\tau$:
\begin{multline*}
\forall {\bm \Delta} \in \mathbb{C} : Q({\mathbf w}^\star + {\bm \Delta}) - Q({\mathbf w}^\star) \\ - \langle \nabla Q({\mathbf w}^\star), {\bm \Delta} \rangle  \geq \kappa \| {\bm \Delta} \|_2^2 - \tau \| {\bm \Delta} \|_1.
\end{multline*}
\end{definition}

\subsection{Proof of Theorem 1}
The proof technique is similar to the one in \cite{cite21}, with the exception that the allowed perturbation set $\mathbb{C}_s$ now depends on $s$.
Let ${\mathcal F} := Q({\mathbf w}^\star + {\bm \Delta}) - Q({\mathbf w}^\star)$, and 
\begin{equation} \label{ConeDef}
\mathbb{C}_s({\mathbf w}^\star) := \{ {\mathbf w} - {\mathbf w}^\star : \forall i ~ {w}_i \geq 0, \| {\mathbf w} \|_1 \leq s \}.
\end{equation}
For the sake of brevity, we will use $\mathbb{C}_s$ to refer to this set in the rest of the paper. We first show that if $\mathcal{F} > 0$ for all ${\bm \Delta} \in \mathbb{K}_{\delta} := \mathbb{C}_s \cap \{ {\bm \Delta} : \| {\bm \Delta} \|_2 = \delta \} $, then $\| \widehat{\mathbf w} - {\mathbf w}^{\star} \|_2 \leq \delta$.
First note that for any ${\bm \Delta}_0 = ({\mathbf w}_0 - {\mathbf w}^\star) \in \mathbb{C}_s$, and $0 \leq a \leq 1$, $a {\bm \Delta}_0$ also belongs to ${\mathbb{C}}_s$:
$$ a {\bm \Delta}_0 = a ({\mathbf w}_0 - {\mathbf w}^\star) =  a {\mathbf w}_0 + (1-a) {\mathbf w}^\star - {\mathbf w}^\star.$$
But we have 
$\| a {\mathbf w}_0 + (1-a) {\mathbf w}^\star \|_1 \leq a \| {\mathbf w}_0 \|_1 + (1-a) \| {\mathbf w}^\star \|_1 \leq s  $. Moreover all elements of $a {\mathbf w}_0 + (1-a) {\mathbf w}^\star $ are positive. Hence $a {\bm \Delta}_0$ belongs to ${\mathbb{C}}_s$.

Next suppose ${\bm \Delta}_0 := \widehat{\mathbf w} - {\mathbf w}^\star$. Note that $ {\bm \Delta}_0 \in \mathbb{C}_s $. We will assume $\| {\bm \Delta}_0 \|_2 > \delta$ and show that for at least one element of $\mathbb{C}_s \cap \mathbb{K}_{\delta}$, $\mathcal{F} \leq 0$. Let ${\bm \Delta}_1 = \frac{\delta}{\| {\bm \Delta}_0 \|} {\bm \Delta}_0$. Based on the earlier argument, it is clear that ${\bm \Delta}_1 \in \mathbb{C}_s$. Furthermore, it belongs to $\mathbb{K}_{\delta}$.
Next, as $\mathcal{F}({\bm \Delta})$ is convex, Jensen's inequality could be used to show that $\mathcal{F}({\bm \Delta}_1) \leq 0$:
\begin{multline*}
 \mathcal{F}\left({\bm 0} + \frac{\delta}{\| {\bm \Delta}_0 \|_2} {\bm \Delta}_0 \right) \leq \left( 1 - \frac{\delta}{\| {\bm \Delta}_0 \|_2} \right) \mathcal{F}({\bm 0}) \\ + \frac{\delta}{\| {\bm \Delta}_0 \|_2} \mathcal{F}({\bm \Delta}_0) \leq 0.
 \end{multline*}
The last part follows by noting that $\mathcal{F}({\bm 0}) = 0$ and $\mathcal{F}({\bm \Delta}_0) \leq 0$ as $\widehat{\mathbf w}$ is the minimizer of $Q$.

Hence to bound the $\ell_2$ error by $\delta$, we need to show that $\mathcal{F}$ is positive all across $\mathbb{C}_s \cap \mathbb{K}_{\delta}$. But, using Lemma \ref{SCCond}, $\mathcal{F}$ could be lower bounded as follows:
$$ \mathcal{F} \geq \kappa \| {\bm \Delta} \|_2^2 - \tau \| {\bm \Delta} \|_1 + \langle \nabla Q({\mathbf w}^\star), {\bm \Delta} \rangle $$
where $\kappa$ and $\tau$ are defined in Lemma \ref{SCCond}. We also have
$$ |\langle \nabla Q({\mathbf w}^\star), {\bm \Delta} \rangle| \leq \| \nabla Q({\mathbf w}^\star) \|_{\infty} \| {\bm \Delta} \|_1 \leq \nu_n \| {\bm \Delta} \|_1,$$
where the last step has been derived using Lemma \ref{GradUp} and $\nu_n$ is defined therein. 
Putting all these together:
\begin{equation} \label{Last}
\mathcal{F} \geq \kappa \| {\bm \Delta} \|_2^2 - 2 (\tau + \nu_n)\sqrt{k} \| {\bm \Delta} \|_2,
\end{equation}
where in the last step we have used the fact that:
\begin{multline*}
 \| {\bm \Delta} \|_1 = \| {\bm \Delta}_S \|_1 + \| {\bm \Delta}_{S^c} \|_1 \stackrel{(i)}{\leq} 2 \| {\bm \Delta}_S \|_1 \\\stackrel{(ii)}{\leq} 2 \sqrt{k} \| {\bm \Delta}_S \|_2 \leq 2 \sqrt{k} \| {\bm \Delta} \|_2 ,
 \end{multline*}
where $S = \supp({\mathbf w}^\star)$.
$(i)$ follows because for any member of $\mathbb{C}_s$, $\| {\bm \Delta}_S \|_1 \geq \| {\bm \Delta}_{S^c} \|_1$ (check Lemma \ref{ConeProp}), and $(ii)$ holds because for any vector ${\mathbf x}$, $\| {\mathbf x} \|_1 \leq \sqrt{dim({\mathbf x})} \| {\mathbf x} \|_2$. 
Now using Eq. \eqref{Last}, it is easy to see if $\| {\bm \Delta} \| \geq \delta := 3 (\tau + \nu_n) \sqrt{k}/\kappa$, $\mathcal{F} > 0$. This completes the proof. $\square$

\begin{lemma} \label{ConeProp}
For any ${\bm \Delta} \in \mathbb{C}_s$ defined in Eq. \eqref{ConeDef}, $\| {\bm \Delta}_S \|_1 \geq {\bm \Delta}_S^c \|_1$, where $S = \supp({\mathbf w}^\star)$.
\end{lemma}
\begin{proof}
From the definition of $\mathbb{C}_s$ in Eq. \eqref{ConeDef}, we know:
$$\|{\bm \Delta}_S\|_1=\|{\mathbf w}_S-{\mathbf w}^\star\|_1\geq \|{\mathbf w}^\star\|_1-\| {\mathbf  w}_S\|_1$$
Moreover, from $\|{\mathbf w}\|_1 \leq s$ for all ${\bm \Delta} \in \mathbb{C}_s$, we have:
$$\|{\mathbf w}_S\|_1+\| {\mathbf w}_{S^c}\|_1=\|{\mathbf w}\|_1\leq s=\|{\mathbf w}^\star\|_1$$
Therefore,
$$\|{\bm \Delta}_{S^c}\|_1=\|{\mathbf w}_{S^c}\|_1\leq \|{\mathbf w}^\star\|_1-\|{\mathbf w}_S\|_1 \leq \|{\bm \Delta}_S\|_1$$
\end{proof}

\begin{lemma} \label{SCCond}
Let $Q$ be the negative log likelihood of the Poisson model introduced in Sec. \ref{PS}, and $\max_{i,j} a_{i,j} \leq a_{\max}$. If ${\mathbf A}$ satisfies $RE(k, \gamma_k)$,
$Q$ will be Strong Convex around ${\mathbf w}^\star$ over the perturbation set $\mathbb{C}_{s}({\mathbf w}^\star)$ (defined in Eq. \ref{ConeDef}) for 
$$\kappa := \frac{\gamma_k}{9 \lambda_{\max} },$$ $$\tau := a^2_{\max} (4 + 2 \log(\lambda_{\max} / \lambda_{\min})) \sqrt{\frac{\log(2/\zeta)}{n \overline{\lambda}_h}}. $$ with probability of at least $1 - \zeta$, where $\zeta \geq 2 \exp\left(\frac{-n \lambda_{\min} \min(1, \lambda_{\min})}{4 \overline{\lambda}_h}\right)$.
\end{lemma}
\begin{proof}
Let
$$ \delta Q := Q({\mathbf w}^\star + {\bm \Delta}) - Q({\mathbf w}^\star) - \langle \nabla Q({\mathbf w}^\star), {\bm \Delta} \rangle,$$ 
and ${\bm \Delta} := {\mathbf w} - {\mathbf w}^\star \in \mathbb{C}_s$.
Using some algebraic manipulations, $\delta Q$ could be written as follows:
$$ \delta Q = \frac{1}{n} \sum_{i = 1}^n -y_i \log \left( 1 + {\mathbf a}_i^{\top} {\bm \Delta} / \lambda_i({\mathbf w}^\star) \right) + y_i {\mathbf a}_i^{\top} {\bm \Delta} / \lambda_i({\mathbf w}^\star). $$
Leting $d_i := -\log \left( 1 + {\mathbf a}_i^{\top} {\bm \Delta} / \lambda_i({\mathbf w}^\star) \right) + {\mathbf a}_i^{\top} {\bm \Delta} / \lambda_i({\mathbf w}^\star)$,
it can be simplified to:
\begin{align*}
\delta Q = \underbrace{\frac{1}{n} \sum_{i = 1}^n - \lambda_i({\mathbf w}^\star) d_i}_{\delta Q_1}  + \underbrace{\frac{1}{n} \sum_{i = 1}^n (y_i - \lambda_i({\mathbf w}^\star) ) d_i }_{\delta Q_2}.
\end{align*}
Assuming that ${\mathbf A}$ satisfies RE$(k, \gamma_k)$, $\delta Q_1$ can be lower bounded for all { ${\bm \Delta} \in \mathbb{C}_{s}({\mathbf w}^\star)$} using Lemma \ref{Q1Lemma}:
$$\delta Q_1 \geq \frac{\gamma_k \| {\bm \Delta} \|^2_2}{9 \lambda_{\max}}.$$
On the other hand, absolute value of $\delta Q_2$ could be shown to be upper bounded as:
\begin{align*}
 |\delta Q_2| \leq & \left(\max_{i, {\bm \Delta} \in \mathbb{C}_{s}}  \underbrace{-\lambda_i({\mathbf w}^\star) \log \left( 1 + \frac{{\mathbf a}_i^{\top} {\bm \Delta}}{\lambda_i({\mathbf w}^\star)} \right)  + {{\mathbf a}_i^{\top} {\bm \Delta}}}_{\delta Q_{2,1}} \right) \\ & \times \left( \underbrace{\frac{1}{n} \sum_{i = 1}^{n} |y_i/\lambda_i({\mathbf w}^\star) - 1|}_{\delta Q_{2, 2}} \right) 
 \end{align*}
Suppose ${\mathbf a}_i^\top {\bm \Delta} \leq 0 $. Then, $\delta Q_{2, 1}$ can be upper bounded as follows:
$$ \delta Q_{2, 1} \leq \left| (\lambda_i({\mathbf w}) - {\mathbf a}_i^\top {\bm \Delta} ) \log\left(1 + \frac{{\mathbf a}_i^\top {\bm \Delta}}{\lambda_i({\mathbf w}^\star)}\right) \right| + \left| {{\mathbf a}_i^\top {\bm \Delta}} \right|.$$
The bound could be simplified to obtain:
\begin{multline*}
\delta Q_{2, 1} \leq \left| \lambda_i({\mathbf w}) \log\left(1 + \frac{{\mathbf a}_i^\top {\bm \Delta}}{\lambda_i({\mathbf w}^\star)}\right) \right| \\ + \left| {\mathbf a}_i^\top {\bm \Delta}  \log\left(1 + \frac{{\mathbf a}_i^\top {\bm \Delta}}{\lambda_i({\mathbf w}^\star)}\right) \right| + \left| {{\mathbf a}_i^\top {\bm \Delta}} \right|.
\end{multline*}

Then, the first order Taylor expansion of the first term around zero yields:
\begin{multline*}
 \delta Q_{2, 1} \leq \lambda_i({\mathbf w}) \left| \frac{{\mathbf a}_i^\top {\bm \Delta}}{\lambda_i({\mathbf w}^\star)} \times \frac{1}{1 + \tilde{c}_i} \right| \\ + a_{\max} {(1 + \log(\lambda_{\max} / \lambda_{\min})) \| {\bm \Delta} \|_1}, 
 \end{multline*}
where $\tilde{c}_i$ is a number between $\frac{{\mathbf a}_i^\top {\bm \Delta}}{\lambda_i({\mathbf w}^\star)}$ and $0$. But note that
$$ \frac{{\mathbf a}_i^{\top} {\bm \Delta}}{\lambda_i({\mathbf w}^\star)} = \frac{\lambda_i({\mathbf w})}{\lambda_i({\mathbf w}^\star)} - 1 \leq 0,$$
therefore, 
$$ 1 + \tilde{c}_i \geq \frac{\lambda_i({\mathbf w})}{\lambda_i({\mathbf w}^\star)}.$$
Hence, we have
$$ \delta Q_{2, 1} \leq {(2 a_{\max} + a_{\max} \log(\lambda_{\max} / \lambda_{\min})) \| {\bm \Delta} \|_1}.$$
Now suppose ${\mathbf a}_i^\top {\bm \Delta} \geq 0$. We can then upper bound $\delta Q_{2, 1}$ as follows:
$$ \delta Q_{2, 1} \leq \left|  \lambda_i({\mathbf w}^\star) \log\left(1 + \frac{{\mathbf a}_i^\top {\bm \Delta}}{\lambda_i({\mathbf w}^\star)}\right) \right| + \left| {{\mathbf a}_i^\top {\bm \Delta}} \right|.$$
Using the Taylor series expansion of the first term, we obtain:
$$ \delta Q_{2, 1} \leq \left|  \lambda_i({\mathbf w}^\star) \frac{{\mathbf a}_i^\top {\bm \Delta}}{\lambda_i({\mathbf w}^\star)} \times \frac{1}{1 + \tilde{c}_i} \right| + \left| {{\mathbf a}_i^\top {\bm \Delta}} \right|,$$
where $\tilde{c}_i$ is a number between $0$ and $\frac{{\mathbf a}_i^\top {\bm \Delta}}{\lambda_i({\mathbf w}^\star)}$. Hence we get:
$\delta Q_{2, 1} \leq 2 a_{\max} \| {\bm \Delta} \|_1$. 
Considering both cases, we arrive at:
$$ \delta Q_{2, 1} \leq {(2 a_{\max} + a_{\max} \log(\lambda_{\max} / \lambda_{\min})) \| {\bm \Delta} \|_1}.$$
On the other hand, using Corollary \ref{CorConc}, $\delta Q_{2,2}$ is upper bounded with probability of at least $1 - \zeta$ :
$$ \delta Q_{2,2} \leq 2 a_{\max} \sqrt{\frac{\log(2/\zeta)}{n \overline{\lambda}_h}}. $$
Using the last two inequalities completes the proof.
\end{proof}

\begin{lemma} \label{Q1Lemma}
If $$\delta Q_1 := \frac{1}{n} \sum_{i = 1}^n - \lambda_i({\mathbf w}^\star) \log \left( 1 + {\mathbf a}_i^{\top} {\bm \Delta} / \lambda_i({\mathbf w}^\star) \right) + {\mathbf a}_i^{\top} {\bm \Delta},$$ and ${\mathbf A}$ satisfies RE$(k, \gamma_k)$,
then $\delta Q_1 \geq \frac{\gamma_k \| {\bm \Delta} \|_2^2 }{9 \lambda_{\max}} $ for all ${\bm \Delta} \in \mathbb{C}_s({\mathbf w}^\star)$.
\end{lemma}
\begin{proof}

We use the following inequality to lower bound $\delta Q_1$ :
\begin{equation}
a \leq b \Rightarrow a \log (1 + x/a) \leq b \log (1 + x/b)
\end{equation}
Hence
\begin{align}
\label{bound_f}
\delta Q_1 \geq &\frac{1}{n} \sum_{i=1}^n- \lambda_{\max} \log\left(1+\frac{{\mathbf a}_i^{\T} {\bm \Delta}}{\lambda_{\max}}\right)+{{\mathbf a}_i^{\T} {\bm \Delta}} \\
= & \frac{\lambda_{\max}}{n} \sum_{i=1}^n- \log\left(1+\frac{{\mathbf a}_i^{\T} {\bm \Delta}}{\lambda_{\max}}\right)+\frac{{\mathbf a}_i^{\T} {\bm \Delta}}{\lambda_{\max}}
\end{align}
and from inequality $-\log(1+x)\geq -x$, we can show that for all ${\bm \Delta}$, $\delta Q_1 \geq 0$. 
We make the following change of variables:
$$X_i=\frac{{\mathbf a}_i^{\T} {\bm \Delta}}{\lambda_{\max}}$$
Since by Lemma \ref{ConeProp}, all ${\bm \Delta} \in \mathbb{C}_s$ satisfy the precondition of the lower bound in RE condition, we have:
$$\frac{1}{n}\|X\|_2^2=\frac{1}{n}\sum_{i=1}^n X^2_i=\frac{1}{n}\sum_{i=1}^n \frac{({\mathbf a}_i^{\T} {\bm \Delta})^2}{\lambda^2_{\max}}\geq \frac{\gamma_k \| {\bm \Delta} \|_2^2}{\lambda^2_{\max}} $$
Now, by applying Taylor series expansion around $X_i=0$ to each term in $\delta Q_1$, we will have:
$$-\log\left(1+X_i\right)+X_i=-X_i+\frac{1}{(1+\widetilde{X}_i)^2}X^2_i+X_i$$
where $|\widetilde{X}_i|$ lies between 0 and $|X_i|$:
$$|\widetilde{X}_i|=|c\times 0+ (1-c)\times X_i|\leq |X_i|\leq \frac{2sa_{\max}}{\lambda_{\max}}$$ 
where the last inequality follows from the fact that $\| {\bm \Delta} \|_1\leq 2s$ for all {${\bm \Delta} \in \mathbb{C}_s$.}Therefore, $\delta Q_1$ can be lower bounded as follows:

$$ \delta Q_1 \geq\min_{\|X\|^2 \geq \gamma_k \frac{n \| {\bm \Delta} \|_2^2}{\lambda^2_{\max}}}\frac{1}{n}\lambda_{\max}\sum_{i=1}^n\frac{1}{(1+\frac{2sa_{\max}}{\lambda_{\max}})^2}X^2_i.$$
Using some algebra, we obtain:
$$ \delta Q_1 \geq \frac{\gamma_k \| {\bm \Delta} \|_2^2}{9 \lambda_{\max}}. $$
\end{proof}

\begin{lemma} \label{GradUp}
The following bound holds with probability of at least $1 - \zeta$
$$ \| \nabla Q({\mathbf w}^\star) \|_{\infty} \leq  \nu_n,$$ where $$\nu_n := {2a_{\max}} \sqrt{\frac{\log(2/\zeta)}{n \overline{\lambda}_h}}, $$
where $e \geq 2 \exp\left(\frac{-n \lambda_{\min} \min(1, \lambda_{\min}) }{4 \overline{\lambda}_h}\right)$.
\end{lemma}
\begin{proof}
The derivative of $Q$ could be stated as
$$ \nabla Q({\mathbf w}^\star) = \frac{1}{n} \sum_{i = 1}^n -\frac{y_i {\mathbf a}_i}{\lambda_i({\mathbf w}^\star)} +  {\mathbf a}_i.$$
Hence 
$$ \| \nabla Q({\mathbf w}^\star) \|_{\infty} \leq \frac{a_{\max}}{n} \sum_{i = 1}^n {| y_i/\lambda_i({\mathbf w}^\star) - 1 |}. $$
Therefore, using Corollary \ref{CorConc}, with probability of at least $1 - \zeta$:
$$ \| \nabla Q({\mathbf w}^\star) \|_{\infty} \leq  {2a_{\max}} \sqrt{\frac{\log(2/\zeta)}{n \overline{\lambda}_h}}, $$
where $\zeta \geq 2 \exp\left(\frac{-n \lambda_{\min} \min(1, \lambda_{\min}) }{4 \overline{\lambda}_h }\right)$.
\end{proof}

\begin{lemma} \label{Concentration}
{\bf Bernstein Inequality} : Let $y_i$, $1 \leq i \leq n$, be independent random samples with means $\lambda_1, \ldots, \lambda_n$. Suppose $\exists L > 0$ such that $\forall m \in \mathbb{N}$
$$ \E[| y_i - \lambda_i|^m] \leq \frac{1}{2} \E[(y_i - \lambda_i)^2] L^{m-2} m! $$
Then the average absolute deviations from the means can be bounded with high probability:
\begin{multline*}
\Pr \left(\frac{1}{n} \sum_{i=1}^n | y_i - \lambda_i | \leq \frac{2t}{n} \sqrt{\sum_{i = 1}^n \E[ (y_i - \lambda_i)^2 ]} \right) \\ \geq 1 - 2\exp(-t^2),
\end{multline*}
for $0 < t \leq \frac{\sqrt{\sum_{i = 1}^n \E[(y_i - \lambda_i)^2] }}{2L}$.
\end{lemma}

\begin{lemma} \label{PoisConc}
For $y_i \sim \Poisson(\lambda_i)$, there exists a number $L > 0$ such that for all integers $m > 1$, 
$$ \E[| y_i - \lambda_i|^m] \leq \frac{1}{2} \E[(y_i - \lambda_i)^2] L^{m-2} m! $$
Moreover, $L = \max(1, \sqrt{\lambda_i})$.
\end{lemma}
\textbf{Remark: }The proof of this Lemma is partly provided in \cite{cite14} and is based on the fact that moment generating function for Poisson distribution with rate $\lambda$,
$$M(t) = \exp\left(\lambda(\exp(t)-1) - \lambda t \right)$$
is an analytic function, which means that its Taylor series converges around $t = 0$ on an open set in $\mathbb{R}$. Therefore, the $k$-th coefficient of Taylor series exists and is bounded. If $\lambda \leq 1$, it could be shown that all Taylor coefficients of $M(t)$ are less than or equal to $\frac{1}{2}$. If $\lambda > 1$, we replace $t$ by $t/\sqrt{\lambda}$. Then, all Taylor coefficients could be shown to be less than $\frac{1}{2}$. This completes the proof. $\square$

\begin{corollary} \label{CorConc}
Let $y_i \sim \Poisson(\lambda_i({\mathbf w}^\star))$, $1 \leq i \leq n$, and $\overline{\lambda}_{h}$ be the harmonic average of $\lambda_i({\mathbf w}^\star)$'s. Then the following bound holds
$$ \Pr \left( \frac{1}{n} \sum_{i = 1}^n | y_i / \lambda_i({\mathbf w}^\star) - 1 | \leq 2 \sqrt{\frac{\log(2/\zeta)}{n \overline{\lambda}_h}} \right) \geq 1 - \zeta$$
for $\zeta \geq 2 \exp\left(-\frac{n \lambda_{\min} \min(1, \lambda_{\min} ) }{4 \overline{\lambda}_h}\right)$.
\end{corollary}
\begin{proof}
Follows trivially from Lemma \ref{PoisConc}, and the Bernstein inequality. 
\end{proof}

\subsection{Proof of Corollary 1}
Note that for any feasible ${\mathbf w}$ in the optimization, $\lambda_i({\mathbf w})$ could be written as:
\begin{multline*}
 \lambda_i({\mathbf w}) = \lambda_{0, i} + {\mathbf a}^{\top}_i {\mathbf w} = \lambda_{0, i} + {\mathbf a}^{(g) \top}_{i} {\mathbf w} + a_{\wedge} s \\ = (\lambda_{0, i} + a_{\wedge} s) + {\mathbf a}^{(g) \top}_{i} {\mathbf w},
 \end{multline*}
where ${\mathbf a}^{(g)}_i$ is the $i$-th row of ${\mathbf A}_g$. Now assuming the new design matrix is ${\mathbf A}_g$ and the base rates are $\lambda_{0, i} + s a_{\wedge}$, we can repeat the proof of the theorem 1. This time, as ${\mathbf A}_g$ satisfies RE$(\mathcal{O}(1), k)$ for $n \geq C k \log p$, we get $\gamma_k = \mathcal{O}(1)$, and $a_{\max}$ has to be replaced by $\max(a_{\vee}, a_{\wedge})$, which is the maximum possible absolute value of ${\mathbf A}_g$ \hbox{entries. $\square$}

\subsection{Proof of Theorem 2} \label{Thm2Pr}


Without loss of generality we assume that the minimum element of ${\mathbf A}$ occurs in the $k$-th column. 
First, using Gilbert-Varshamov bound, we obtain $M \geq \exp\{(k-1)/8\}$ vectors ${\bm \tau}_i \in \mathbb{R}^k$ on the hypercube $\{0,\,1\}^k$ such that the Hamming distance is bounded from below by $(k-1)/4$ \cite{Gun11}, i.e.,
$$
d_H({\bm \tau}_i, {\bm \tau}_j) \geq (k-1)/4
$$

We choose vectors ${\mathbf w}_j$, $1 \leq j \leq M$ as follows:
$$
{\mathbf w}_j = \left(s- \frac{1}{c} \sqrt{{a_{\min} s \over n \eta}} (k - 1) \right ) \mathbf{e}_k + \frac{1}{c} \sqrt{{a_{\min} s \over n \eta}} \sum_{t=1}^{k-1} \tau_j(t) \mathbf{e}_t
$$
where $\mathbf{e}_i \in \mathbb{R}^p$ are unit vectors with one in the $i$-th component, and $c$ is a constant. We assume that $s \geq \frac{2}{c} \sqrt{{a_{\min} s \over n \eta}} (k - 1)$, namely, the number of measurements $n$ is sufficiently large. Note that for all $j$, $\| {\mathbf w}_j \|_1 \leq s$, and $\lambda_{0,t} + {\mathbf a}_i^\top {\mathbf w}_j \geq a_{\min} s/2$.

For future reference, we note down $\ell_2$ and $\ell_1$ norms of pairwise difference of the ${\mathbf w}_i$'s:
\begin{align*}
\|\bm \delta_{ij}\|_2 := \|{\mathbf w}_i - {\mathbf w}_j\|_2 & =  \frac{1}{c} \sqrt{{a_{\min} s \over n \eta} d_H({\bm \tau}_i, {\bm \tau}_j)} \\ & \geq \frac{1}{2c} \sqrt{a_{\min} s (k - 1) \over n \eta}  
\end{align*}
\begin{align*}
\|\bm \delta_{ij}\|_2 \leq \frac{1}{c} \sqrt{\frac{(k-1) a_{\min} s}{n \eta}}
\end{align*}
\begin{align*}
\|\bm \delta_{ij}\|_1 & = \|{\mathbf w}_i - {\mathbf w}_j\|_1 \geq \frac{k-1}{4c} \sqrt {{a_{\min} s \over n \eta} } 
\end{align*}

We next compute the KL divergence for Poisson distributed random variables namely, $y_t \sim \text{Pois}(\lambda_{0,t} + {\mathbf a}_t^{\top} {\mathbf w}_i)$ for different pairs $i,\,j$. It follows that,
\begin{align*}
I := {1 \over M^2} & \sum_{t=1}^n \sum_{i,j} D_{KL}(\text{Pois}(\lambda_{0,t} + {\mathbf a}_t^{\top} {\mathbf w}_i) || \text{Pois}(\lambda_{0,t} + {\mathbf a}_t^{\top} {\mathbf w}_j) ) \\ & = {1 \over M^2} \sum_{t=1}^n \sum_{i,j} {1 \over 2} ({\mathbf a}_t^{\top} {\bm \delta}_{ij}) \log \left ( 1 + \frac{ {\mathbf a}_t^{\top} {\bm \delta}_{ij}}{\lambda_{0,t} + {\mathbf a}_t^\top {\mathbf w}_j} \right )
\end{align*}
where we have assumed w.l.o.g that ${\mathbf a}_t^{\top} {\bm \delta}_{ij} > 0$. Next note that,
$$
\log \left ( 1 + { {\mathbf a}_t^{\top} {\bm \delta}_{ij} \over (\lambda_{0,t} + {\mathbf a}_t^\top {\mathbf w}_j)} \right ) \leq { {\mathbf a}_t^{\top} {\bm \delta}_{ij} \over (\lambda_{0,t} + {\mathbf a}_t^\top {\mathbf w}_j)}
$$
Consequently, we obtain
\begin{align*}
I & \leq  \sum_{t=1}^n {({\mathbf a}_t^\top {\bm \delta}_{ij})^2 \over 
2 (\lambda_{0,t} + {\mathbf a}_t^\top {\mathbf w}_j)} \leq n {{1 \over n} \|{\mathbf A} {\bm \delta}_{ij}\|_2^2 \over 2 a_{\min} s/2} \leq \frac{n \eta \| {\bm \delta}_{i,j} \|_2^2}{a_{\min} s} \\ & \leq \frac{(k-1)}{c^2}
\end{align*}
Finally, from the generalized Fano bound \cite{Bin97} we know that
\begin{align*}
\inf_{\widehat{\mathbf w}} & \sup_{{\mathbf w}^\star \in \mathcal{G}} \E(\| \widehat{\mathbf w} - {\mathbf w}^\star \|_2)  \geq \frac{1}{2} \min_{i \neq j} \| {\bm \delta}_{i,j} \|_2 \left (1 - \frac{I + \log 2}{\log(M)} \right ) 
\end{align*}
If we make $c$ large enough ($c \geq 34$), we will have
$$ \frac{I + \log 2}{\log(M)} \leq \frac{8}{c^2} + \frac{8 \log 2}{k - 1} \leq 0.7,$$
for $k \geq 9$.
As a result, we obtain
\begin{align*}
\inf_{\widehat{\mathbf w}} & \sup_{{\mathbf w}^\star \in \mathcal{G}} \E(\| \widehat{\mathbf w} - {\mathbf w}^\star \|_2)  \geq \frac{0.3}{4 c} \sqrt{\frac{(k - 1) a_{\min} s}{n \eta}} 
\end{align*}

\bibliographystyle{IEEEbib}
\bibliography{refs}


%
%
\IEEEpeerreviewmaketitle

\end{document}